\documentclass[12pt]{article}

\usepackage[top=1in, bottom=1in, left=0.9in, right=0.9in]{geometry}

\usepackage{times}

\date{}

\title{\bfseries Active Ranking with Subset-wise Preferences}

\author{
Aadirupa Saha\thanks{Indian Institute of Science, Bangalore, India. {\tt aadirupa@iisc.ac.in}}, \and Aditya Gopalan \thanks{Indian Institute of Science, Bangalore, India. {\tt aditya@iisc.ac.in} }
}

\usepackage{microtype}
\usepackage{cancel}
\usepackage{graphicx}
\usepackage{booktabs} 
\usepackage{amssymb}
\usepackage{color}
\usepackage{amsmath}
\usepackage{amsthm}
\usepackage{thmtools}
\usepackage{thm-restate}
\usepackage{algorithm}
\usepackage{algorithmic}

\usepackage{natbib}
\usepackage{hyperref}
\usepackage{nameref}

\newtheorem{thm}{Theorem}
\newtheorem{lem}[thm]{Lemma}

\newtheorem{cor}[thm]{Corollary}
\newtheorem{defn}[thm]{Definition}

\newtheorem{rem}{Remark}

\newcommand{\R}{{\mathbb R}}

\newcommand{\N}{{\mathbb N}}

\newcommand{\E}{{\mathbf E}}

\newcommand{\1}{{\mathbf 1}}

\newcommand{\cA}{{\mathcal A}}

\newcommand{\cB}{{\mathcal B}}

\newcommand{\cC}{{\mathcal C}}

\newcommand{\cE}{{\mathcal E}}

\newcommand{\cF}{{\mathcal F}}
\newcommand{\cG}{{\mathcal G}}
\newcommand{\cT}{{\mathcal T}}

\newcommand{\cN}{{\mathcal N}}

\newcommand{\cD}{{\mathcal D}}

\newcommand{\cR}{{\mathcal R}}

\newcommand{\X}{{\mathcal X}}

\newcommand{\tS}{{\tilde {S}}}
\newcommand{\sS}{{S^*}}

\newcommand{\hp}{{\hat p}}
\newcommand{\p}{{\mathbf p}}

\newcommand{\sm}{{\setminus}}
\def \algupdt{{\it Rank-Break}}

\def \algwin{{\it Find-the-Pivot}}
\def \algant{{\it Beat-the-Pivot}} 
\def \algestscr{{\it Score-and-Rank}}

\def \pac{{\bf $(\epsilon,\delta)$-PAC-Rank}}
\def \pacdb{{\bf $(\epsilon,\delta)$-PAC-Rank-Multiplicative}}

\def \br{{\it Best-Ranking}}
\def \ebi{{\it $\epsilon$-Best-Item}}
\def \ebr{{\it $\epsilon$-Best-Ranking}}
\def \ebrm{{\it $\epsilon$-Best-Ranking-Multiplicative}}

\newcommand{\btheta}{\boldsymbol \theta}
\newcommand{\bSigma}{\boldsymbol \Sigma}

\newcommand{\bnu}{{\boldsymbol \nu}}
\newcommand{\htheta}{{\hat{\theta}}}
\newcommand{\bsigma}{\boldsymbol \sigma}

\begin{document}

\maketitle

\begin{abstract}

We consider the problem of probably approximately correct (PAC) ranking $n$ items by adaptively eliciting subset-wise preference feedback. At each round, the learner chooses a subset of $k$ items and observes stochastic feedback indicating preference information of the winner (most preferred) item of the chosen subset drawn  according to a Plackett-Luce (PL) subset choice model unknown a priori. The objective is to identify an $\epsilon$-optimal ranking of the $n$ items with probability at least $1 - \delta$. %
When the feedback in each subset round is a single Plackett-Luce-sampled item, we show $(\epsilon, \delta)$-PAC algorithms with a sample complexity of $O\left(\frac{n}{\epsilon^2} \ln \frac{n}{\delta} \right)$ rounds, which we establish as being order-optimal by exhibiting a matching sample complexity lower bound of $\Omega\left(\frac{n}{\epsilon^2} \ln \frac{n}{\delta} \right)$---this shows that there is essentially no improvement possible from the pairwise comparisons setting ($k = 2$). When, however, it is possible to elicit top-$m$ ($\leq k$) ranking feedback according to the PL model from each adaptively chosen subset of size $k$, we show that an $(\epsilon, \delta)$-PAC ranking sample complexity of $O\left(\frac{n}{m \epsilon^2} \ln \frac{n}{\delta} \right)$ is achievable with explicit algorithms, which represents an $m$-wise reduction in sample complexity compared to the pairwise case. This again turns out to be order-wise unimprovable across the class of symmetric ranking algorithms. Our algorithms rely on a novel {pivot trick} to maintain only $n$ itemwise score estimates, unlike $O(n^2)$ pairwise score estimates that has been used in prior work. We report results of numerical experiments that corroborate our findings. 
\end{abstract}

\vspace*{-20pt}

\section{Introduction}

Ranking or sorting is a classic search problem and basic algorithmic primitive in computer science. %
 Perhaps the simplest and most well-studied ranking problem is using (noisy) pairwise comparisons, which started from the work of \citet{FeiRag94:NoisySort}, and which has recently been studied in machine learning under the rubric of ranking in `dueling bandits' \citep{Busa14survey}.

However, more general {\em subset-wise} preference feedback arises naturally in application domains where there is flexibility to learn by eliciting preference information from among a {\em set} of offerings, rather than by just asking for a pairwise comparison. For instance, web search and recommender systems applications typically involve users expressing preferences by clicking on one result (or a few results) from a presented set. Medical surveys, adaptive tutoring systems and multi-player sports/games are other domains where subsets of questions, problem set assignments and tournaments, respectively, can be carefully crafted to learn users' relative preferences by subset-wise feedback. 

In this paper, we explore {\em active}, probably approximately correct (PAC) ranking of $n$ items using subset-wise, preference information. We assume that upon choosing a subset of $k \geq 2$ items, the learner receives preference feedback about the subset according to the well-known Plackett-Luce (PL) probability model \citep{Marden_book}. The learner faces the goal of returning a near-correct ranking of all items, with respect to a tolerance parameter $\epsilon$ on the items' PL weights, with probability at least $1-\delta$ of correctness, after as few subset comparison rounds as possible. 
%
%
In this context, we make the following contributions: 

\begin{enumerate}
	\item We consider active ranking with winner information feedback, where the learner, upon playing a subset $S_t \subseteq [n]$ of exactly $k = |S_t|$ elements at each round $t$, receives as feedback a single winner sampled from the Plackett-Luce probability distribution on the elements of $S_t$.  
	We design two $(\epsilon,\delta)$-{PAC} algorithms for this problem (Section \ref{sec:algo_wi}) with sample complexity $O\left(\frac{n}{\epsilon^2}\ln \frac{n}{\delta}\right)$ rounds, for learning a near-correct ranking on the items. 
 
	\item  We show a matching lower bound of $\Omega\left( \frac{n}{\epsilon^2} \ln \frac{n}{\delta}\right)$ rounds on the $(\epsilon,\delta)$-{PAC} sample complexity of ranking with winner information feedback (Section \ref{sec:lb_wi}), which is also of the same order as that for the dueling bandit ($k=2$) \citep{BTM}. This implies that despite the increased flexibility of playing larger sets, with just winner information feedback, one cannot hope for a faster rate of learning than in the case of pairwise comparisons.
		 
	\item In the setting where it is possible to obtain `top-rank' feedback -- an ordered list of $m \leq k$ items sampled from the Plackett-Luce distribution on the chosen subset -- we show that natural generalizations of the winner-feedback algorithms above achieve $(\epsilon,\delta)$-{PAC}  sample complexity of $O\left(\frac{n}{m \epsilon^2}\ln \frac{n}{\delta}\right)$ rounds (Section \ref{sec:res_fr}), which is a significant improvement over the case of only winner information feedback. We show that this is order-wise tight by exhibiting a matching $\Omega\left( \frac{n}{m \epsilon^2} \ln \frac{n}{\delta}\right)$ lower bound on the sample complexity across $(\epsilon,\delta)$-{PAC} algorithms. 
		
	\item We report numerical results to show the performance of the proposed algorithms on synthetic environments (Section \ref{sec:expts}).
%
 
\end{enumerate}

By way of techniques, the PAC algorithms we develop leverage the property of independence of irrelevant attributes (IIA) of the Plackett-Luce model, which allows for $O(n)$ dimensional parameter estimation with tight confidence bounds, even in the face of a combinatorially large number of possible subsets of size $k$. We also  devise a generic `pivoting' idea in our algorithms to efficiently estimate a global ordering using only local comparisons with a pivot or probe element: split the entire pool into playable subsets all containing one common element, learn local orderings relative to this element and then merge. Here again, the IIA structure of the PL model helps to ensure consistency among preferences aggregated across disparate subsets but with a common reference pivot. Our sample complexity lower bounds are information-theoretic in nature and rely on a generic change-of-measure argument but with carefully crafted confusing instances.

\vspace*{2pt}
\noindent \textbf{Related Work.} 
Over the years, ranking from pairwise preferences ($k=2$) has been studied in both the batch or non-adaptive setting \cite{Gleich+11,Rajkumar+16,Wauthier+13,Negahban+12} and the active or adaptive setting \cite{Mossel08,Jamieson11,Ailon12}.  
In particular, prior work has addressed the problem of statistical parameter estimation given preference observations from the Plackett-Luce model in the offline setting \cite{Negahban+12,SueIcml+15,KhetanOh16,Hajek+14}. 
There also have been recent developments on the PAC objective for different pairwise preference models, such as those satisfying stochastic triangle inequalities and strong stochastic transitivity \citep{BTM}, general utility-based preference models \citep{SAVAGE}, the Plackett-Luce model \citep{Busa_pl} and the Mallows model \citep{Busa_mallows}]. Recent work has studied PAC-learning objectives other than identifying the single (near) best arm, e.g. recovering a few of the top arms \citep{Busa_top,MohajerIcml+17,ChenSoda+17}, or the true ranking of the items \citep{Busa_aaai,falahatgar_nips}. There is also work on the problem of Plackett-Luce parameter estimation in the subset-wise feedback setting \cite{SueIcml+17,KhetanOh16}, but for the batch (offline) setup where the sampling is not adaptive. Recent work by \citet{ChenSoda+18} analyzes an active learning problem in the Plackett-Luce model with subset-wise feedback; however, the objective there is to recover the top-$\ell$ (unordered) items of the model, unlike full-rank recovery considered in this work. Moreover, they give instance-dependent sample complexity bounds, whereas we allow a tolerance ($\epsilon$) in defining good rankings, natural in many settings \cite{Busa_pl,BTM,Busa_mallows}. 


\section{Preliminaries}
\label{sec:prelims}
{\bf Notation.} We denote the set $[n] = \{1,2,...,n\}$.
When there is no confusion about the context, we often represent (an unordered) subset $S$ as a vector, or ordered subset, $S$ of size $|S|$ (according to, say, the order induced by the natural global ordering $[n]$ of all the items). In this case, $S(i)$ denotes the item (member) at the $i$th position in subset $S$.   
$\bSigma_S = \{\bsigma \mid \bsigma$  is a permutation over items of $ S\}$. where for any permutation $\bsigma \in \Sigma_{S}$, $\sigma(i)$ denotes the position of element $i \in S$ in the ranking $\bsigma$.
$\1(\varphi)$ denote an indicator variable that takes the value $1$ if the predicate $\varphi$ is true, and $0$ otherwise. 
$Pr(A)$ is used to denote the probability of event $A$, in a probability space that is clear from the context.
$Ber(p) \text{ and } Geo(p)$ respectively denote Bernoulli and Geometric \footnote{this is the `number of trials before success' version} random variable with probability of success at each trial being $p \in [0,1]$. Moreover, for any $n \in \N$, $Bin(n,p)$ and $\emph{NB}(n,p)$ respectively denote Binomial and Negative Binomial distribution.

\subsection{Discrete Choice Models and Plackett-Luce (PL)}
\label{sec:RUM}  

A discrete choice model specifies the relative preferences of two or more discrete alternatives in a given set. %
%
A widely studied class of discrete choice models is the class of {\it Random Utility Models} (RUMs), which assume a ground-truth utility score $\theta_{i} \in \R$ for each alternative $i \in [n]$, and assign a conditional distribution $\cD_i(\cdot|\theta_{i})$ for scoring item $i$. To model a winning alternative given any set $S \subseteq [n]$, one first draws a random utility score $X_{i} \sim \cD_i(\cdot|\theta_{i})$ for each alternative in $S$, and selects an item with the highest random score. 

%

One widely used RUM is the {\it Multinomial-Logit (MNL)} or {\it Plackett-Luce model (PL)}, where the $\cD_i$s are taken to be independent Gumbel distributions with parameters $\theta'_i$ \citep{Az+12}, i.e., with probability densities $\cD_i(x_{i}|\theta'_i) = e^{-(x_j - \theta'_j)}e^{-e^{-(x_j - \theta'_j)}}$, $\theta'_i \in R, ~ \forall i \in [n]$. Moreover assuming $\theta'_i = \ln \theta_i$, $\theta_i > 0 ~\forall i \in [n]$, it can be shown in this case the probability that an alternative $i$ emerges as the winner in the set $S \ni i$ becomes:
$
Pr(i|S) = \frac{{\theta_i}}{\sum_{j \in S}{\theta_j}}.
$

Other families of discrete choice models can be obtained by imposing different probability distributions over the utility scores $X_i$, e.g. if $(X_1,\ldots X_n) \sim \cN(\btheta,\boldsymbol \Lambda)$ are jointly normal with mean $\btheta = (\theta_1,\ldots \theta_n)$ and covariance $\boldsymbol \Lambda \in \R^{n \times n}$, then the corresponding RUM-based choice model reduces to the {\it Multinomial Probit (MNP)}. 
 
\textbf{Independence of Irrelevant Alternatives}
\label{sec:iia}
A choice model $Pr$ is said to possess the {\it Independence of Irrelevant Attributes (IIA)} property if the ratio of probabilities of choosing any two items, say $i_1$ and $i_2$ from within any choice set $S \ni {i_1,i_2}$ is independent of a third alternative $j$ present in $S$ \citep{IIA-relevance16}. Specifically,
$
\frac{Pr(i_1|S_1)}{Pr(i_2|S_1)} = \frac{Pr(i_1|S_2)}{Pr(i_2|S_2)}$ for any two distinct subsets $S_1,S_2 \subseteq [n]
$ 
that contain $i_1$ and $i_2$. Plackett-Luce satisfies the IIA property. 

\section{Problem Setup}
\label{sec:prb_setup}

We consider the PAC version of the sequential decision-making problem of finding the {ranking} of $n$ items by making subset-wise comparisons. %
Formally, the learner is given a finite set $[n]$ of $n > 2$ arms. At each decision round $t = 1, 2, \ldots$, the learner selects a subset $S_t \subseteq [n]$ of $k$  items, and receives (stochastic) feedback about the winner (or most preferred) item of $S_t$ drawn from a Plackett-Luce (PL) model with parameters $\btheta = (\theta_1,\theta_2,\ldots, \theta_n)$, a priori unknown to the learner. The nature of the feedback is described in Section \ref{sec:feed_mod}. We assume henceforth that $\theta_i \in [0,1], \,\forall i \in [n]$, and also $1 = \theta_1 > \theta_2 > \ldots > \theta_n$ for ease of  exposition\footnote{We naturally assume that this knowledge ordering of the items is not known to the learning algorithm, and note that extension to the case where several items have the same highest parameter value is easily accomplished.}. 
 

 

\begin{defn}[{\ebi}]
\label{def:pacbest_item}
For any $\epsilon \in [0,1)$, an item $i$ is called \ebi\, 
if its PL score parameter $\theta_i$ is worse than the \emph{Best-Item} $i^* = 1$ by no more than $\epsilon$, i.e. if $\theta_i \ge \theta_1 - \epsilon$. A $0$-best item is an item with largest PL parameter, which is also a Condorcet winner \citep{Ramamohan+16} in case it is unique. 
\end{defn}




\begin{defn}[{\ebr}] 
\label{def:pac_best_ranking}
We define a ranking $\bsigma \in \Sigma_{[n]}$ to be an \ebr\, when no pair of items in $[n]$ is misranked by $\bsigma$ unless their PL scores are $\epsilon$-close to each other. Formally, 
$
\nexists i,j \in [n], \text{ such that } ~\sigma(i) > \sigma(j) \text{ and } \theta_i \ge \theta_j + \epsilon. 
$ 
A $0$-Best-Ranking will be called a \br\, or \emph{optimal ranking} of the PL model. With $1 = \theta_1 > \theta_2 > \ldots > \theta_n$, clearly the unique \br\, is $\bsigma^* = (1, 2, \ldots, n)$.
\end{defn}


\begin{defn}[{\ebrm}] 
\label{def:pac_best_ranking_mult}
We define a ranking $\bsigma \in \Sigma_{[n]}$ of $\bsigma^*$ to be \ebrm\, if
$
\nexists i,j \in [n], \text{ such that } ~\sigma(i) > \sigma(j), \text{ with } Pr(i|\{i,j\}) \ge \frac{1}{2} + \epsilon. 
$
\end{defn}

Note: The term `{multiplicative}' emphasizes the fact that the condition $Pr(i|\{i,j\}) \ge \frac{1}{2} + \epsilon$ equivalently imposes a multiplicative constraint ${\theta_i} \ge {\theta_j}\bigg(\frac{1/2 + \epsilon}{1/2 - \epsilon}\bigg)$ on the PL score parameters.

\subsection{Feedback models}
\label{sec:feed_mod}
By feedback model, we mean the information received (from the `environment') once the learner plays a subset $S \subseteq [n]$ of $k$ items. We consider the following feedback models in this work:

\textbf{Winner of the selected subset (WI):} 
The environment returns a single item $I \in S$, drawn independently from the probability distribution
$
\label{eq:prob_win}
Pr(I = i|S) = \frac{{\theta_i}}{\sum_{j \in S} \theta_j} ~~\forall i \in S.
$

\textbf{Full ranking on the selected subset (FR):} The environment returns a full ranking $\bsigma \in \bSigma_{S}$, drawn from the probability distribution
$
\label{eq:prob_rnk1}
Pr(\bsigma |S) = \prod_{i = 1}^{|S|}\frac{{\theta_{\sigma^{-1}(i)}}}{\sum_{j = i}^{|S|}\theta_{\sigma^{-1}(j)}}, \; \sigma \in \bSigma_S.
$ 
This is equivalent to picking item $\bsigma^{-1}(1) \in S$ according to winner (WI) feedback from $S$, then picking $\bsigma^{-1}(2)$ according to WI feedback from $S \setminus \{\bsigma^{-1}(1)\}$, and so on, until all elements from $S$ are exhausted, or, in other words, successively sampling $|S|$ winners from $S$ according to the PL model, without replacement. But more generally, one can define

\textbf{Top-$m$ ranking from the selected subset (TR-$m$ or TR):} The environment successively samples (without replacement) only the first $m$ items from among $S$, according to the PL model over $S$, and returns the ordered list. 
It follows that {\bf TR} reduces to {\bf FR} when $m=k=|S|$ and to {\bf WI} when $m = 1$.

\subsection{Performance Objective: \pac \,-- Correctness and Sample Complexity} 
\label{sec:obj}

Consider a problem instance with Plackett-Luce (PL) model parameters $\btheta \equiv (\theta_1, \ldots, \theta_n)$ and subsetsize $k \leq n$, with its \br\, being $\bsigma^* = (1,2,\ldots n)$, and $\epsilon,\delta \in (0,1)$ are two given constants. 
A sequential algorithm that operates on this problem instance, with WI feedback model, is said to be \pac\, if (a) it stops and outputs a ranking $\bsigma \in \bSigma_{[n]}$ after a finite number of decision rounds (subset plays) with probability $1$, and (b) the probability that its output $\bsigma$ is an \ebr\, is at least $1-\delta$, i.e, $Pr(\bsigma\text{ is \ebr}) \geq 1-\delta$. Furthermore, by {\em sample complexity} of the algorithm, we mean the expected time (number of decision rounds) taken by the algorithm to stop. 

{In the context of our above problem objective, it is worth noting the work by \cite{Busa_pl} addressed a similar problem, except in the dueling bandit setup ($k = 2$) with the same objective as above, except with the notion of \ebrm---we term this new objective as \pacdb \, as referred later for comparing the results}. The two objectives are however equivalent under a mild boundedness assumption as follows:

\begin{restatable}[]{lem}{paceqv}
\label{lem:pacobj_eqv}
Assume $\theta_i \in [a,b],\, \forall i \in [n]$, for any $a,b \in (0,1)$. If an algorithm is \pac, then it is also $(\epsilon',\delta)$-\textbf{\textup{PAC-Rank-Multiplicative}} for any $\epsilon' \le \frac{\epsilon}{4b}$. On the other hand, if an algorithm is \pacdb, then it is also $(\epsilon',\delta)$-\textbf{\textup{PAC-Rank}} for any $\epsilon' \le 4a\epsilon(1+\epsilon)$.
\end{restatable}

\section{Parameter Estimation with PL based preference data}
\label{sec:est_pl_score}

We develop in this section some useful parameter estimation techniques based on adaptively sampled preference data from the PL model, which will form the basis for our PAC algorithms later on, in Section \ref{sec:algant}.



\subsection{\hspace{0pt} Estimating Pairwise Preferences via Rank-Breaking.}
\label{sec:est_pp}

\emph{Rank breaking} is a well-understood idea involving the extraction of pairwise comparisons from (partial) ranking data, and then building pairwise estimators on the obtained pairs by treating each comparison independently \citep{KhetanOh16,SueIcml+17}, e.g., a winner $a$ sampled from among ${a,b,c}$ is rank-broken into the pairwise preferences $a \succ b$, $a \succ c$. We use this idea to devise estimators for the pairwise win probabilities $p_{ij} = P(i|\{i,j\}) = \theta_i/(\theta_i + \theta_j)$ in the active learning setting. The following result, used to design Algorithm \ref{alg:alg_ant} later, establishes explicit confidence intervals for pairwise win/loss probability estimates under adaptively sampled PL data. 


\begin{restatable}[Pairwise win-probability estimates for the PL model]
{lem}{plsimulator}
\label{lem:pl_simulator}
Consider a Plackett-Luce choice model with parameters $\btheta = (\theta_1,\theta_2, \ldots, \theta_n)$, and fix two  items $i,j \in [n]$. Let $S_1, \ldots, S_T$ be a sequence of (possibly random) subsets of $[n]$ of size at least $2$, where $T$ is a positive integer, and $i_1, \ldots, i_T$ a sequence of random items with each $i_t \in S_t$, $1 \leq t \leq T$, such that for each $1 \leq t \leq T$, (a) $S_t$ depends only on $S_1, \ldots, S_{t-1}$, and (b) $i_t$ is distributed as the Plackett-Luce winner of the subset $S_t$, given $S_1, i_1, \ldots, S_{t-1}, i_{t-1}$ and $S_t$, and (c) $\forall t: \{i,j\} \subseteq S_t$ with probability $1$. Let $n_i(T) = \sum_{t=1}^T \1(i_t = i)$ and $n_{ij}(T) = \sum_{t=1}^T \1(\{i_t \in \{i,j\}\})$. Then, for any positive integer $v$, and $\eta \in (0,1)$,
\begin{align*}
& Pr\left( \frac{n_i(T)}{n_{ij}(T)} - \frac{\theta_i}{\theta_i + \theta_j} \ge \eta, \; n_{ij}(T) \geq v \right) \leq e^{-2v\eta^2},\\
& Pr\left( \frac{n_i(T)}{n_{ij}(T)} - \frac{\theta_i}{\theta_i + \theta_j} \le -\eta, \; n_{ij}(T) \geq v \right) \leq e^{-2v\eta^2}. 
\end{align*}
\end{restatable}

\subsection{Estimating relative PL scores ($\theta_i/\theta_j$) using Renewal Cycles}
\label{sec:est_thet}

We detail another method to directly estimate (relative) score parameters of the PL model, using renewal cycles and the IIA property.  
%
%


\begin{restatable}[]{lem}{geomgf}
\label{lem:geo2mgf}
Consider a Plackett-Luce choice model with parameters $(\theta_1,\theta_2, \ldots, \theta_n)$, $n \geq 2$, and an item $b \in [n]$. Let $i_1, i_2, \ldots$ be a sequence of iid draws from the model. Let $\tau = \min\{t \ge \N \mid i_t = b\}$ be the first time at which $b$ appears, and for each $i \neq b$, let $w_i(\tau) = \sum_{t=1}^\tau \1(i_t = i)$ be the number of times $i \neq b$ appears until time $\tau$. Then, $\tau - 1$ and $w_i(\tau)$ are Geometric random variables with parameters $\frac{\theta_b}{\sum_{j \in [n]} \theta_j}$ and $\frac{\theta_b}{\theta_i+\theta_b}$, respectively.

\end{restatable}

With this in hand, we now show how fast the empirical mean estimates over several renewal cycles (defined by the appearance of a distinguished item) converge to the true relative scores $\frac{\theta_i}{\theta_b}$, a result to be employed in the design of Algorithm \ref{alg:alg_estscr} later. 

\begin{restatable}[Concentration of Geometric Random Variables via the Negative Binomial distribution.]{lem}{geoconc}
\label{lem:geoconc}
Suppose $X_1, X_2, \ldots X_d$ are $d$ iid Geo$(\frac{\theta_b}{\theta_b + \theta_i})$ random variables, and $Z = \sum_{i = 1}^{d}X_i$. Then, for any $\eta > 0$,
$
Pr\Big( \Big|\frac{Z}{d} - \frac{\theta_i}{\theta_b}\Big| \ge \eta \Big) < 2\exp\Bigg( - \frac{2d\eta^2}{\Big( 1+\frac{\theta_i}{\theta_b} \Big)^2\Big(\eta + 1+\frac{\theta_i}{\theta_b}  \Big)}\Bigg)$.
\end{restatable}

\section{Algorithms for WI Feedback}
\label{sec:algo_wi}

This section describes the design of \pac\, algorithms which use winner information (WI) feedback. 

A key idea behind our proposed algorithms is to estimate the relative strength of each item with respect to a fixed item, termed as a \emph{pivot-item} $b$. This helps to compare every item on common terms (with respect to the pivot item) even if two items are not directly compared with each other. Our first algorithm \algant\, maintains pairwise score estimates $P_{ib}$ of the items $i \in [n]\sm\{b\}$ with respect to the pivot element, based on the idea of \emph{Rank-Breaking} and Lemma \ref{lem:pl_simulator}. 
The second algorithm \algestscr\, directly estimates the relative scores $\frac{\theta_i}{\theta_b}$ for each item $i \in [n]\sm\{b\}$, relying on Lemma \ref{lem:geo2mgf} (Section \ref{sec:est_thet}).
Once all item scores are estimated with enough confidence, the items are simply sorted with respect to their preference scores to obtain a ranking. 


\subsection{The \algant\, algorithm} 
\label{sec:algant}

\begin{center}
\begin{algorithm}[htbp]
   \caption{\textbf{\algant} }
   \label{alg:alg_ant}
\begin{algorithmic}[1]
   \STATE {\bfseries Input:} 
   \STATE ~~~ Set of item: $[n]$ ($n \ge k$), and subset size: $k$
   \STATE ~~~ Error bias: $\epsilon >0$, confidence parameter: $\delta >0$
   \STATE {\bfseries Initialize:} 
   \STATE ~~~ $\epsilon_b \leftarrow \min(\frac{\epsilon}{2},\frac{1}{2})$; $b \leftarrow $ \algwin($n,k,\epsilon_b,\frac{\delta}{2}$)
   \STATE ~~~ Set $S \leftarrow [n]\setminus \{b\}$, and divide $S$ into $G: = \lceil \frac{n-1}{k-1} \rceil$ sets $\cG_1, \cG_2, \cdots \cG_G$ such that $\cup_{j = 1}^{G}\cG_j = S$ and $\cG_{j} \cap \cG_{j'} = \emptyset, ~\forall j,j' \in [G], \, |G_j| = (k-1),\, \forall j \in [G-1]$
   \STATE ~~ \textbf{If} $|\cG_{G}| < (k-1)$, \textbf{then} set $\cR \leftarrow \cG_G$, and $S \leftarrow S\setminus \cR$, $S' \leftarrow $ Randomly sample $(k - 1 -|\cG_G|)$ items from $S$, and set $\cG_G \leftarrow \cG_G \cup S'$
   \STATE ~~~ \textbf{Set} $\cG_j = \cG_j \cup \{b\}, \, \forall j \in [G]$
   \FOR {$g = 1,2, \ldots, G$}
   \STATE Set $\epsilon' \leftarrow \frac{\epsilon}{16}$ and $\delta' \leftarrow \frac{\delta}{8n}$
   \STATE Play $\cG_g$ for $t:= \frac{2k}{\epsilon'^2}\log \frac{1}{\delta'}$ times 
   \STATE Set $w_i \leftarrow$ Number of times $i$ won in $m$ plays of $\cG_g$, and estimate $\hp_{ib} \leftarrow \frac{w_{i}}{w_{i} + w_{b}}, \, \forall i \in \cG_g$ 
   \ENDFOR
   \STATE Choose $ \bsigma \in \bSigma_{[n]}$, such that $\sigma(b) = 1$ and $\bsigma(i) < \bsigma(j)$ if $\hp_{ib} > \hp_{jb}, \, \forall i,j \in S\cup \cR$ 
   \STATE {\bfseries Output:} The ranking $\bsigma \in \bSigma_{[n]}$ 
\end{algorithmic}
\end{algorithm}
\vspace{2pt}
\end{center}

\algant\, (Algorithm \ref{alg:alg_ant}) first estimates an \emph{approximate Best-Item} $b$ with high probability $(1-\delta/2)$. We do this using the subroutine \algwin$(n,k,\epsilon,\delta)$ (Algorithm \algwin) that with probability at least $(1-\delta)$ \algwin\, outputs an $\epsilon$-\emph{Best-Item} within a sample complexity of $O(\frac{n}{\epsilon^2} \log \frac{k}{\delta})$. 

Once the best item $b$ is estimated, \algant\, divides the rest of the $n-1$ items into groups of size $k-1$, $\cG_1, \cG_2, \cdots \cG_G$, and appends $b$ to each group. This way elements of every group get to compete (and hence compared) against $b$, which aids estimating the pairwise score compared to the pivot item $b$, $\hat p_{i b}$ owing to the \emph{IIA property} of PL model and Lemma \ref{lem:pl_simulator} (Sec. \ref{sec:est_pp}), sorting which we obtain the final ranking. Theorem \ref{thm:batt_giant} shows that \algant\, enjoys the optimal sample complexity guarantee of $O\Big ( (\frac{n}{\epsilon^2})\log \Big( \frac{n}{\delta} \Big) \Big))$. 
The pseudo code of \algant\, is given in Algorithm \ref{alg:alg_ant}.

\begin{restatable}[\algant:  Correctness and Sample Complexity]{thm}{ubpiv}
\label{thm:batt_giant}
\algant\, (Algorithm \ref{alg:alg_ant}) is \pac\, with {sample complexity} $O\big(\frac{n}{\epsilon^2} \log \frac{n}{\delta}\big)$. 
\end{restatable}

\subsection{The \algestscr\, algorithm} 
\label{sec:algestscr}

\algestscr\, (Algorithm \ref{alg:alg_estscr}) differs from \algant\, in terms of the score estimate it maintains for each item.
Unlike our previous algorithm, instead of maintaining pivot-preference scores $p_{ib} = Pr(i \succ b)$, \algant, aims to directly estimate the PL-score $\theta_i$ of each item relative to score of the pivot $\theta_b$. In other words, the algorithm maintains the \emph{relative score} estimates $\frac{\theta_i}{\theta_b}$ for every item $i \in [n]\setminus \{b\}$ borrowing results from Lemma \ref{lem:geo2mgf} and \ref{lem:geoconc}, and finally return the ranking sorting the items with respect to their \emph{relative pivotal-score}. \algestscr\, also runs within an optimal sample complexity of $\big( \frac{n}{\epsilon^2}\ln \frac{n}{\delta} \big)$ as shown in Theorem \ref{thm:est_thet}.
The complete algorithm is described in Algorithm \ref{alg:alg_estscr}. 


\begin{center}
\begin{algorithm}[H]
   \caption{\textbf{\algestscr} }
   \label{alg:alg_estscr}
\begin{algorithmic}[1]
   \STATE {\bfseries Input:} 
   \STATE ~~~ Set of item: $[n]$ ($n \ge k$), and subset size: $k$
   \STATE ~~~ Error bias: $\epsilon >0$, confidence parameter: $\delta >0$
   \STATE {\bfseries Initialize:} 
   \STATE ~~~ $\epsilon_b \leftarrow \min(\frac{\epsilon}{2},\frac{1}{2})$, $b \leftarrow $ \algwin($n,k,\epsilon_b,\frac{\delta}{4}$)
   \STATE ~~~ Set $S \leftarrow [n]\setminus \{b\}$, and divide $S$ into $G: = \lceil \frac{n-1}{k-1} \rceil$ sets $\cG_1, \cG_2, \cdots \cG_G$ such that $\cup_{j = 1}^{G}\cG_j = S$ and $\cG_{j} \cap \cG_{j'} = \emptyset, ~\forall j,j' \in [G], \, |G_j| = (k-1),\, \forall j \in [G-1]$
   \STATE ~~~ \textbf{If} $|\cG_{G}| < (k-1)$, \textbf{then} set $\cR \leftarrow \cG_G$, and $S \leftarrow S\setminus \cR$, $S' \leftarrow $ Randomly sample $(k - 1 -|\cG_G|)$ items from $S$, and set $\cG_G \leftarrow \cG_G \cup S'$
   \STATE ~~~ Set $\cG_j = \cG_j \cup \{b\}, \, \forall j \in [G]$
   \FOR {$g = 1,2, \ldots, G$}
   \STATE Set $\epsilon' \leftarrow \frac{\epsilon}{24}$ and $\delta' \leftarrow \frac{\delta}{8n}$ 
   \REPEAT
   \STATE Play $\cG_g$ and observe the winner.
   \UNTIL{$b$ is chosen for $t = \frac{1}{\epsilon'^2}\ln \frac{1}{\delta'} $ times} 
   \STATE Set $w_i \leftarrow$ the total number of wins of item $i$ in $\cG_g$, and $\hat \theta_i^b \leftarrow \frac{w_i}{t}, \, \forall i \in \cG_g\setminus\{b\}$ 
   \ENDFOR
   \STATE Choose $ \bsigma \in \bSigma_{[n]}$, such that $\sigma(b) = 1$ and $\bsigma(i) < \bsigma(j)$ if $\htheta_{i}^b > \htheta_{j}^b, \, \forall i,j \in S\cup \cR$ 
   \STATE {\bfseries Output:} The ranking $\bsigma \in \bSigma_{[n]}$ 
\end{algorithmic}
\end{algorithm}
\vspace{2pt}
\end{center}

\begin{restatable}[\algestscr:  Correctness and Sample Complexity]{thm}{ubthet}
\label{thm:est_thet}
\algestscr\, (Algorithm \ref{alg:alg_estscr}) is \pac\, with {sample complexity} $O\big(\frac{n}{\epsilon^2} \log \frac{n}{\delta}\big)$. 
\end{restatable}

\subsection{The \algwin\, subroutine (for Algorithms \ref{alg:alg_ant} and \ref{alg:alg_estscr})}
\label{app:alg_bi}


In this section, we describe the {pivot selection procedure} \algwin$(n,k,\epsilon,\delta)$. The algorithm serves the purpose of finding an \ebi\, with high probability $(1-\delta)$ that is used as \emph{the pivoting element} $b$ both by Algorithm \ref{alg:alg_ant} (Sec. \ref{sec:algant}) and \ref{alg:alg_estscr} (Sec. \ref{sec:algestscr}). %
%

\algwin\, is based on the simple idea of tracing the empirical best item--specifically, it maintains a running winner $r_\ell$ at every iteration $\ell$, which is made to compete with a set of $k-1$ arbitrarily chosen other items long enough ($\frac{2k}{\epsilon^2}\ln \frac{2n}{\delta}$ rounds). At the end if the empirical winner $c_\ell$ turns out to be more than $\frac{\epsilon}{2}$-favorable than the running winner $r_\ell$, in term of its pairwise preference score: $\hp_{c_\ell,r_\ell} > \frac{1}{2} + \frac{\epsilon}{2}$, then $c_\ell$ replaces $r_\ell$, or else $r_\ell$ retains its place and status quo ensues. The process recurses till we are left with only a single element which is returned as the pivot. The formal description of \algwin\, is in Algorithm \ref{alg_bi}.

\begin{center}
\begin{algorithm}[H]
   \caption{\textbf{\algwin} subroutine ~(for Algorithms \ref{alg:alg_ant} and \ref{alg:alg_estscr})}
   \label{alg_bi}
\begin{algorithmic}[1]
   \STATE {\bfseries Input:} 
   \STATE ~~~ Set of items: $[n]$, Subset size: $n \geq k > 1$
   \STATE ~~~ Error bias: $\epsilon >0$, confidence parameter: $\delta >0$
   \STATE {\bfseries Initialize:} 
   \STATE ~~~ $r_1 \leftarrow $  Any (random) item from $[n]$, $\cA \leftarrow $ Randomly select $(k-1)$ items from $[n]\setminus \{r_1\}$
   \STATE ~~~ Set $\cA \leftarrow \cA \cup \{r_1\}$, and $S \leftarrow [n]\setminus \cA$  
   \WHILE {$\ell = 1, 2, \ldots$}
   	\STATE Play the set $\cA$ for $t:= \frac{2k}{\epsilon^2}\ln \frac{2n}{\delta}$ rounds
   \STATE $w_i \leftarrow$ \# ($i$ won in $t$ plays of $\cA$), $\forall i \in \cA$
   \STATE $c_\ell \leftarrow \underset{i \in \cA}{\text{argmax}}~w_i$; $\hp_{ij} \leftarrow \frac{w_i}{w_i + w_j}, \, \forall i,j \in \cA, i \neq j$
   \STATE \textbf{if} $\hp_{c_\ell,r_\ell} > \frac{1}{2} + \frac{\epsilon}{2} $: $r_{\ell+1} \leftarrow c_\ell$; \textbf{else} $r_{\ell+1} \leftarrow r_\ell$ 	
	\IF {$(S == \emptyset)$}
	\STATE Break (exit the while loop)    
	\ELSIF {$|S| < k-1$}
	\STATE $\cA \leftarrow$ Select $(k-1-|S|)$ items from $\cA \setminus \{r_\ell\}$ uniformly at random, $\cA \leftarrow \cA \cup \{r_\ell\} \cup S$; $S \leftarrow \emptyset$
	\ELSE
	\STATE $\cA	\leftarrow $ Select $(k-1)$ items from $S$ uniformly at random, $\cA \leftarrow \cA \cup \{r_\ell\}$; $S\leftarrow S \setminus \cA$
	\ENDIF
	\ENDWHILE
   \STATE {\bfseries Output:} The item $r_\ell$ 
\end{algorithmic}
\end{algorithm}
\vspace{-2pt}
\end{center}




\begin{restatable}[\algwin:  Correctness and Sample Complexity with WI]{lem}{ubtrc}
\label{lem:trace_best}
\algwin\, (Algorithm \ref{alg_bi}) achieves the $(\epsilon,\delta)$-{PAC} objective with sample complexity $O(\frac{n}{\epsilon^2} \log\frac{n}{\delta})$.
\end{restatable}

\vspace*{-10pt}

\section{Lower Bound}
\label{sec:lb_wi}
\vspace*{-5pt}
In this section we show the minimum sample complexity required for any \emph{symmetric algorithm} to be \pac\, is at least $\Omega\Big (\frac{n}{\epsilon^2}\log  \frac{n}{\delta}  \Big)$ (Theorem \ref{thm:lb_plpac_win}). 
Note this in fact matches the sample complexity bounds of our proposed algorithms (recall Theorem \ref{thm:batt_giant} and \ref{thm:est_thet}) showing the tightness of both our upper and lower bound guarantees. The key observation lies in noting that results are independent of $k$, which shows the learning problem with $k$-subsetwise WI feedback is as hard as that of the dueling bandit setup $(k=2)$---the flexibility of playing a $k$ sized subset does not help in faster information aggregation. We first define the notion of a \emph{symmetric} or label-invariant algorithm.

\begin{defn}[Symmetric Algorithm]
\label{def:sym_alg}
A PAC algorithm $\cA$ is said to be {\em symmetric} if its output is insensitive to the specific labelling of items, i.e., if for any PL model $(\theta_1, \ldots, \theta_n)$, bijection $\phi: [n] \to [n]$ and ranking $\bsigma: [n] \to [n]$, it holds that $Pr(\cA \text{ outputs } \bsigma \, | \,  (\theta_1, \ldots, \theta_n)) = Pr(\cA \text{ outputs } \bsigma \circ \phi^{-1} \, | \, (\theta_{\phi(1)}, \ldots, \theta_{\phi(n)}))$, where $Pr(\cdot \, | (\alpha_1, \ldots, \alpha_n) )$ denotes the probability distribution on the trajectory of $\cA$ induced by the PL model $(\alpha_1, \ldots, \alpha_n)$. 
\end{defn}


\begin{restatable}[Lower bound on Sample Complexity with WI feedback]{thm}{lbwin}
\label{thm:lb_plpac_win}
Given a fixed $\epsilon \in \big(0,\frac{1}{\sqrt{8}}\big]$, $\delta \in [0,1]$, and a symmetric \pac\, algorithm $\cA$ for WI feedback, there exists a PL instance $\nu$ such that the sample complexity of $\cA$ on $\nu$ is at least
$
\Omega\bigg( \frac{n}{\epsilon^2} \ln \frac{n}{4\delta}\bigg).
$
\end{restatable}

\begin{proof}\textbf{(sketch)}.
The argument is based on the following change-of-measure argument (Lemma  $1$) of \cite{Kaufmann+16_OnComplexity}. (restated in Appendix \ref{app:gar} as Lemma \ref{lem:gar16}). 
To employ this result, note that in our case, each bandit instance corresponds to an instance of the problem with arm set containing all the subsets of $[n]$ of size $k$: $\{S = (S(1), \ldots S(k)) \subseteq [n] ~|~ S(i) < S(j), \, \forall i < j\}$. The key part of our proof relies on carefully crafting a true instance, with optimal arm $1$, and a family of slightly perturbed alternative instances $\{\bnu^a: a \neq 1\}$, each with optimal arm $a \neq 1$.

\textbf{Designing the problem instances.} We first renumber the $n$ items as $\{0,1,2,\ldots n-1\}$.
Now for any integer $q \in [n-1]$, we define $\bnu_{[q]}$ to be the set of problem instances where any instance $\nu_S \in \bnu_{[q]}$ is associated to a set $S \subseteq [n-1]$, such that $|S| = q$, and the PL parameters $\btheta$ associated to instance $\nu_S$ are set up as follows: $\theta_0 = \theta\bigg( \frac{1}{4} - \epsilon^2 \bigg), \theta_j = \theta\bigg( \frac{1}{2} + \epsilon \bigg)^2 \forall j \in S, \text{ and } \theta_j = \theta\bigg( \frac{1}{2} - \epsilon \bigg)^2 \forall j \in [n-1]\setminus S$, %
%
for some $\theta \in \R_+, ~\epsilon > 0$. 
We will restrict ourselves to the class of instances of the form $\bnu_{[q]}, \, q \in [n-1]$.

Corresponding to each problem  $\nu_{S} \in \bnu_{[q]}$, such that $q \in [n-2]$, consider a \emph{slightly altered} problem instance $\nu_{\tS}$ associated with a set $\tS \subseteq [n-1]$, such that $\tS = S\cup \{i\} \subseteq [n-1]$, where $i \in [n-1]\setminus S$. Following the same construction as above, the PL parameters of the problem instance $\nu_{\tS}$ are set up as:
$\theta_0 = \theta\bigg( \frac{1}{4} - \epsilon^2 \bigg), \theta_j = \theta\bigg( \frac{1}{2} + \epsilon \bigg)^2 
\forall j \in \tS, \text{ and } \theta_j = \theta\bigg( \frac{1}{2} - \epsilon \bigg)^2 \forall j \in [n-1]\setminus \tS
$.

\begin{rem}
\label{rem:inst_not}
Note that any problem instance $\nu_S \in \bnu_{[q]}$, $q \in [n-1]$ is thus can be uniquely defined by its underlying set $S \in [n-1]$. For simplicity we will also use the notations $S \in \bnu_{[q]}$ to define the problem instance.
\end{rem}

\begin{rem}
It is easy to verify that, for any $\theta \ge \frac{1}{1-2\epsilon}$, an \ebr\, (Definition. \ref{def:pac_best_ranking}) for problem instance $\nu_{S}, \, S \subseteq [n-1]$, say $\bsigma_{S}$, has to satisfy the following:
$
\sigma_{S}(i) < \sigma_{S}(0), \, \forall i \in S \text{ and } \sigma_{S}(0) < \sigma_{S}(j), \, \forall j \in [n-1]\setminus S
$. 
Thus for any instance $S$, the items in $S$ should precede item $0$ which itself precedes items in $[n-1]\sm S$. 
\end{rem}

For any ranking $\bsigma \in \Sigma_n$, we denote by $\sigma(1:i)$ the set first $i$ items in the ranking, for any $i \in [n]$.

We now fix any set $S^* \subset [n-1]$, $|S^*|=q = \lfloor \frac{n}{2} \rfloor$. Theorem \ref{thm:lb_plpac_win} is now obtained by applying Lemma \ref{lem:gar16} on pair of instances $(\nu_{S^*}, \nu_{\tS^*})$, for all possible choices of $\tS=S\cup\{i\}$, $i \in [n-1]\sm S$, and for the event $\cE:= \{\bsigma_{\cA}(1:q+1) = S^*\cup\{0\}\}$. However we apply a tighter upper bounds for the KL-divergence term of in the right hand side of Lemma \ref{lem:gar16}. It is easy to note that as $\cA$ is \pac\,, obviously $Pr_{S^*}\Big( \bsigma_{\cA}(1:q+1) = S^*\cup\{0\} \Big) > 1-\delta$, and
$
Pr_{\tS^*}\Big( \bsigma_{\cA}(1:q+1)  = S^*\cup\{0\} \Big) < Pr_{\tS^*}\Big( \bsigma_{\cA}(1:q+1) \neq \tS^* \Big) < \delta.
$ %
Further using
$kl(Pr_{\nu_{\sS}}(\cE),Pr_{\nu_{\tS^*}}(\cE)) \ge kl(1-\delta,\delta) \geq \ln \frac{1}{4\delta}$ (due to Lemma \ref{lem:kl_del}) leads to a lower bound guarantee of $\Omega\Big( \frac{n}{\epsilon^2}\ln \frac{1}{\delta} \Big)$, but that is loose by an $\Omega\big(\frac{n}{\epsilon^2} \log n\big)$ additive factor. Novelty of our analysis lies in further utilising the \emph{symmetric property} of $\cA$ to prove a tighter upper bound od the kl-divergence with the following result:

\begin{restatable}[]{lem}{symlem}
\label{lem:lb_sym}
For any symmetric \pac\, algorithm $\cA$, and any problem instance $\nu_S \in \bnu_{[q]}$ associated to the set $S \subseteq [n-1]$, $q \in [n-1]$, and for any item $i \in S$, 
$
Pr_{S}\Big( \bsigma_{\cA}(1:q)  = S \sm \{i\}\cup\{0\} \Big) < \frac{\delta}{q},
$
where $Pr_{S}(\cdot)$ denotes the probability of an event under the underlying problem instance $\nu_S$ and the internal randomness of the algorithm $\cA$ (if any).
\end{restatable}

For our purpose, we use the above result for $S = \tS^*$ which leads to the desired tighter upper bound for $kl(Pr_{\nu_{\sS}}(\cE),Pr_{\nu_{\tS^*}}(\cE)) \ge kl(1-\delta,\frac{\delta}{q}) \geq \ln \frac{q}{4\delta}$, the last inequality follows due to Lemma \ref{lem:kl_del} (Appendix \ref{app:lb_kl}).
The complete proof can be found in Appendix \ref{app:wilb}.
\end{proof}

\vspace*{0pt}

\begin{rem}
\label{rem:wi_null_k}
Theorem \ref{thm:lb_plpac_win} shows, rather surprisingly, that the PAC-ranking with winner feedback information from size-$k$ subsets, does not become easier (in a worst-case sense) with $k$, implying that there is no reduction in hardness of learning from the pairwise comparisons case ($k = 2$). While one may expect sample complexity to improve as the number of items being simultaneously tested in each round ($k$) becomes larger, there is a counteracting effect due to the fact that it is intuitively `harder' for a high-value item to win in just a single winner draw against a (large) population of $k-1$ other competitors. 
A useful heuristic here is that the number of bits of information that a single winner draw from a size-$k$ subset provides is $O(\ln k)$, which is not significantly larger than when $k > 2$; thus, an algorithm cannot accumulate significantly more information per round compared to the pairwise case. 
\end{rem}

We also have a similar lower bound result for the \pacdb\, objective of \citet{Busa_pl} (Section \ref{sec:prb_setup}):

\begin{restatable}{thm}{pac2lb}
\label{thm:pac2_lb}
Given a fixed $\epsilon \in \big(0,\frac{1}{\sqrt{8}}\big]$, $\delta \in [0,1]$, and a symmetric \pacdb\, algorithm $\cA$ for WI feedback model, there exists a PL instance $\nu$ such that the sample complexity of $\cA$ on $\nu$ is at least
$
\Omega\bigg( \frac{n}{\epsilon^2} \ln \frac{n}{4\delta}\bigg).
$
\end{restatable}


\vspace*{-15pt}

\section{Analysis with Top Ranking (TR) feedback}
\label{sec:res_fr}

We now proceed to analyze the problem with \textit{Top-$m$ Ranking} ({TR}) feedback (Section \ref{sec:feed_mod}). We first show that unlike WI feedback, the sample complexity lower bound here scales as $\Omega\bigg( \frac{n}{m\epsilon^2} \ln \frac{n}{\delta}\bigg)$ (Theorem \ref{thm:lb_pacpl_rnk}), which is a factor ${m}$ smaller than that in Theorem \ref{thm:lb_plpac_win} for the WI feedback model. At a high level, this is because TR reveals preference information for $m$ items per feedback round, as opposed to just a single (noisy) information sample of the winning item {(WI)}. Following this, we also present two algorithms for this setting which are shown to enjoy an exact optimal sample complexity guarantee of $O\bigg( \frac{n}{m\epsilon^2} \ln \frac{n}{\delta}\bigg)$ (Section \ref{sec:alg_fr}).

\vspace*{-10pt}

\subsection{Lower Bound for Top-$m$ Ranking (TR) feedback}
\label{sec:lb_fr}


\begin{restatable}[Sample Complexity Lower Bound for TR]{thm}{lbrnk}
\label{thm:lb_pacpl_rnk}
Given $\epsilon \in \Big(0,\frac{1}{8}\Big]$ and $\delta \in (0,1]$, and a symmetric \pac\, algorithm $\cA$ with top-$m$ ranking ({TR}) feedback ($2 \le m \le k$), there exists a PL instance $\nu$ such that the expected sample complexity of $\cA$ on $\nu$ is at least
$\Omega\bigg( \frac{n}{m\epsilon^2} \ln \frac{n}{4\delta}\bigg)
$.
\end{restatable}

\begin{rem}
The sample complexity lower bound for \pac\,  with top-$m$ ranking ({TR}) feedback model is $\frac{1}{m}$-times that of the {WI} model (Theorem \ref{thm:lb_plpac_win}). Intuitively, revealing a ranking on $m$ items in a $k$-set provides about $\ln \left({k \choose m} m!\right) = O(m \ln k)$ bits of information per round, which is about $m$ times as large as that of revealing a single winner, yielding an acceleration by a factor of $m$.  
\end{rem}


\begin{cor}
\label{cor:lb_pacpl_rnk}
Given $\epsilon \in \Big(0,\frac{1}{\sqrt{8}}\Big]$ and $\delta \in (0,1]$, and a symmetric \pac\, algorithm $\cA$ with full ranking ({FR}) feedback ($m = k$), there exists a PL instance $\nu$ such that the expected sample complexity of $\cA$ on $\nu$ is at least
$\Omega\bigg( \frac{n}{k\epsilon^2} \ln \frac{1}{4\delta}\bigg)
$.
\end{cor}

\vspace*{-10pt}

\subsection{Algorithms for Top-$m$ Ranking (TR) feedback model}
\label{sec:alg_fr}

This section presents two algorithms that works on top-$m$ ranking feedback and shown to satisfy the \pac\ property with the optimal sample complexity guarantee of $O\Big( \frac{n}{m\epsilon^2}\frac{n}{\delta}  \Big)$ that matches the lower bound derived in the previous section (Theorem \ref{thm:lb_pacpl_rnk}). This shows a $\frac{1}{m}$ factor faster learning rate compared to the WI feedback model which id achieved by generalizing our earlier two proposed algorithms (see Algorithm \ref{alg:alg_ant} and \ref{alg:alg_estscr}, Sec. \ref{sec:algo_wi} for {WI} feedback) to the top-$m$ ranking ({TR}) feedback.
The two algorithms 
are presented below:

\vspace*{3pt}

\noindent
\subsubsection*{Algorithm \ref{alg:alg_ant_mod}: Generalizing \algant \, for top-$m$ ranking (TR) feedback.}

The first algorithm is based on our earlier \algant\, algorithm (Algorithm \ref{alg:alg_ant}) which essentially maintains the empirical pivotal preferences $\hp_{ib}$ for each item $i \in [n]\sm\{b\} $ by applying a novel trick of \textit{Rank Breaking} on the {TR} feedback (i.e. the ranking $\bsigma \in \Sigma_{S_m}$, $S_m \subseteq [m], |S_m| = m$) received per round after each $k$-subsetwise play. 

\vspace{3pt}
\noindent
\textbf{Rank-Breaking.} \cite{KhetanOh16,AzariRB+14}
The concept of \textit{Rank Breaking} is essentially based upon the clever idea of extracting pairwise comparisons from subsetwise preference information. Formally, given any set $S$ of size $k$, if $\bsigma \in \bSigma_{S_m},\, (S_m \subseteq S,\, |S_m|=m)$ denotes a possible top-$m$ ranking of $S$, the \textit{Rank Breaking} subroutine considers each item in $S$ to be beaten by its preceding items in $\bsigma$ in a pairwise sense. 
%
See Algorithm \ref{alg:updt_win} for detailed description of the procedure.

\vspace*{-5pt}

\begin{center}
\begin{algorithm}[H]
   \caption{\algupdt\, (for updating the pairwise win counts $w_{ij}$ with {TR} feedback for Algorithm \ref{alg:alg_ant_mod})}
   \label{alg:updt_win}
\begin{algorithmic}[1]
   \STATE {\bfseries Input:} 
   \STATE ~~~ Subset $S \subseteq [n]$, $|S| = k$ ($n\ge k$) 
   \STATE ~~~ A top-$m$ ranking $\bsigma \in \bSigma_{S_m}$, $S_m \subseteq [n], \, |S_m| = m$
   \STATE ~~~ Pairwise (empirical) win-count $w_{ij}$ for each item pair $i,j \in S$
   \WHILE {$\ell = 1, 2, \ldots m$}
   	\STATE Update $w_{\sigma(\ell)i} \leftarrow w_{\sigma(\ell)i} + 1$, for all $i \in S \setminus\{\sigma(1),\ldots,\sigma(\ell)\}$
	\ENDWHILE
\end{algorithmic}
\end{algorithm}
\vspace{-5pt}
\end{center}
\vspace*{-5pt}

Of course in general, \emph{Rank Breaking} may lead to arbitrarily inconsistent estimates of the underlying model parameters \citep{Az+12}. However, owing to the {\it IIA property} of the Plackett-Luce model, we get clean concentration guarantees on $p_{ij}$ using Lem. \ref{lem:pl_simulator}. This is precisely the idea used for obtaining the $\frac{1}{m}$ factor improvement in the sample complexity guarantees of \algant\, as analysed in Theorem \ref{thm:batt_giant}.
The formal descriptions of \algant\, generalized to the setting of {TR} feedback, is given in Algorithm \ref{alg:alg_ant_mod}.

\begin{center}
\begin{algorithm}[t]
   \caption{\textbf{\algant} (for {TR} feedback) }
   \label{alg:alg_ant_mod}
\begin{algorithmic}[1]
   \STATE {\bfseries Input:} 
   \STATE ~~~ Set of item: $[n]$ ($n \ge k$), and subset size: $k$
   \STATE ~~~ Error bias: $\epsilon >0$, confidence parameter: $\delta >0$
   \STATE {\bfseries Initialize:} 
   \STATE ~~~ $\epsilon_b \leftarrow \min(\frac{\epsilon}{2},\frac{1}{2})$; $b \leftarrow $ \algwin($n,k,\epsilon_b,\frac{\delta}{2}$)
   \STATE ~~~ Set $S \leftarrow [n]\setminus \{b\}$, and divide $S$ into $G: = \lceil \frac{n-1}{k-1} \rceil$ sets $\cG_1, \cG_2, \cdots \cG_G$ such that $\cup_{j = 1}^{G}\cG_j = S$ and $\cG_{j} \cap \cG_{j'} = \emptyset, ~\forall j,j' \in [G], \, |G_j| = (k-1),\, \forall j \in [G-1]$
   \STATE ~~ \textbf{If} $|\cG_{G}| < (k-1)$, \textbf{then} set $\cR \leftarrow \cG_G$, and $S \leftarrow S\setminus \cR$, $S' \leftarrow $ Randomly sample $(k - 1 -|\cG_G|)$ items from $S$, and set $\cG_G \leftarrow \cG_G \cup S'$
   \STATE ~~~ \textbf{Set} $\cG_j = \cG_j \cup \{b\}, \, \forall j \in [G]$
   \FOR {$g = 1,2, \ldots, G$}
   \STATE Set $\epsilon' \leftarrow \frac{\epsilon}{16}$, $\delta' \leftarrow \frac{\delta}{8n}$ and $t:= \frac{2k}{m\epsilon'^2}\log \frac{1}{\delta'}$
	\STATE Initialize pairwise (empirical) win-count $w_{ij} \leftarrow 0$, for each item pair $i,j \in \cG_g$
	\FOR {$\tau = 1, 2, \ldots t$}
   	\STATE Play the set $\cG_g$ 
   	\STATE Receive feedback: The top-$m$ ranking $\bsigma \in \bSigma_{\cG^\tau_{gm}}$, where $\cG^\tau_{gm} \subseteq \cG_g$, $|\cG^\tau_{gm}| = m$ 
   	\STATE Update win-count $w_{ij}$ of each item pair $i,j \in \cG_g$ using \algupdt$(\cG_g,\bsigma)$
   	\ENDFOR 
   \STATE Estimate $\hp_{ib} \leftarrow \frac{w_{ib}}{w_{ib}+w_{bi}}, \, \forall i \in \cG_g \sm\{b\}$ 
   \ENDFOR
   \STATE Choose $ \bsigma \in \bSigma_{[n]}$, such that $\sigma(b) = 1$ and $\bsigma(i) < \bsigma(j)$ if $\hp_{ib} > \hp_{jb}, \, \forall i,j \in S\cup \cR$ 
   \STATE {\bfseries Output:} The ranking $\bsigma \in \bSigma_{[n]}$ 
\end{algorithmic}
\end{algorithm}
\vspace{2pt}
\end{center}

\vspace*{-15pt}

\begin{restatable}[\algant:  Correctness and Sample Complexity for TR feedback]{thm}{ubantfr}
\label{thm:batt_giant_fr}
With top-$m$ ranking ({TR}) feedback model, \algant\, (Algorithm \ref{alg:alg_ant_mod}) is \pac\, with sample complexity $O(\frac{n}{m\epsilon^2} \log\frac{n}{\delta})$.
\end{restatable}

\begin{rem}
\label{rem:1overm}
Comparing Theorems \ref{thm:batt_giant} and \ref{thm:batt_giant_fr} shows that the sample complexity of \algant\, with TR feedback (Algorithm \ref{alg:alg_ant_mod}) is ${m}$ times smaller than its corresponding counterpart for WI feedback, owing to the additional information gain revealed from preferences among $m$ items instead of just $1$. 
\end{rem}

\vspace{-2pt}
\noindent
\subsubsection*{Algorithm \ref{alg:alg_estscr_mod}: Generalizing \algestscr \, for top-$m$ ranking (TR) feedback.}

The second algorithm essentially goes along the same line of \algestscr\, algorithm (Algorithm \ref{alg:alg_estscr}) except that in this case, after each round of subsetwise play, the empirical win-count $w_i$ of any element $i \in \cG_g\sm \{b\}$, at any group $g \in [G]$, is updated based on the selection of item $i$ in top-$m$ ranking, i.e. $w_i$ is incremented by $1$ as long as item $i$ is selected in the top-$m$ ranking  $\sigma \in \Sigma_{\cG_{gm}^\tau}$, at any round $\tau$ (Line $14$). The reduced sample complexity (compared to the earlier case of WI feedback) is thus achieved, as now it takes much lesser number of plays to select $b$ for $t = \frac{1}{\epsilon'^2}\ln \frac{1}{\delta'} $ times in the top-$m$ ranking (Line $13$). The formal descriptions of \algestscr\, (generalized to the setting of {TR} feedback) is given in Algorithm \ref{alg:alg_estscr_mod}.

\begin{center}
\begin{algorithm}[t]
   \caption{\textbf{\algestscr} (for TR feedback)}
   \label{alg:alg_estscr_mod}
\begin{algorithmic}[1]
   \STATE {\bfseries Input:} 
   \STATE ~~~ Set of item: $[n]$ ($n \ge k$), and subset size: $k$
   \STATE ~~~ Error bias: $\epsilon >0$, confidence parameter: $\delta >0$
   \STATE {\bfseries Initialize:} 
   \STATE ~~~ $\epsilon_b \leftarrow \min(\frac{\epsilon}{2},\frac{1}{2})$, $b \leftarrow $ \algwin($n,k,\epsilon_b,\frac{\delta}{4}$)
   \STATE ~~~ Set $S \leftarrow [n]\setminus \{b\}$, and divide $S$ into $G: = \lceil \frac{n-1}{k-1} \rceil$ sets $\cG_1, \cG_2, \cdots \cG_G$ such that $\cup_{j = 1}^{G}\cG_j = S$ and $\cG_{j} \cap \cG_{j'} = \emptyset, ~\forall j,j' \in [G], \, |G_j| = (k-1),\, \forall j \in [G-1]$
   \STATE ~~~ \textbf{If} $|\cG_{G}| < (k-1)$, \textbf{then} set $\cR \leftarrow \cG_G$, and $S \leftarrow S\setminus \cR$, $S' \leftarrow $ Randomly sample $(k - 1 -|\cG_G|)$ items from $S$, and set $\cG_G \leftarrow \cG_G \cup S'$
   \STATE ~~~ Set $\cG_j = \cG_j \cup \{b\}, \, \forall j \in [G]$
   \FOR {$g = 1,2, \ldots, G$}
   \STATE Set $\epsilon' \leftarrow \frac{\epsilon}{24}$ and $\delta' \leftarrow \frac{\delta}{8n}$ 
   \REPEAT
   \STATE Play $\cG_g$ and observe the feedback: The top-$m$ ranking $\bsigma \in \bSigma_{\cG^\tau_{gm}}$, where $\cG^\tau_{gm} \subseteq \cG_g$, $|\cG^\tau_{gm}| = m$ 
   \UNTIL{$b$ is chosen for $t = \frac{1}{\epsilon'^2}\ln \frac{1}{\delta'} $ times in the top-$m$ ranking} 
   \STATE Set $w_i \leftarrow$ is the total number of times item $i$ appear in top-$m$ rankings in between $t$ selections of $b$, and $\hat \theta_i^b \leftarrow \frac{w_i}{t}, \, \forall i \in \cG_g\setminus\{b\}$ 
   \ENDFOR
   \STATE Choose $ \bsigma \in \bSigma_{[n]}$, such that $\sigma(b) = 1$ and $\bsigma(i) < \bsigma(j)$ if $\htheta_{i}^b > \htheta_{j}^b, \, \forall i,j \in S\cup \cR$ 
   \STATE {\bfseries Output:} The ranking $\bsigma \in \bSigma_{[n]}$ 
\end{algorithmic}
\end{algorithm}
\vspace{2pt}
\end{center}


\vspace*{-30pt}

\section{Experiments}
\label{sec:expts}

The experimental setup of our empirical evaluations are as follows:

\textbf{Algorithms.} We simulate the results on our two proposed algorithms (1). \algant \, and (2). \algestscr.
We also compare our ranking performance with the \emph{PLPAC-AMPR} method, the only existing method (to the best of our knowledge) that addresses the online PAC ranking problem, although only in the dueling bandit setup (i.e. $k = 2$).

\textbf{Ranking Performance Measure.} 
We use the popular pairwise \emph{Kendall's Tau ranking loss} ('pd-loss' in short) \cite{Monjardet98} for measuring the accuracy of the estimated ranking $\bsigma$ with respect to the \br \, $\bsigma^*$ (corresponding to the true PL scores $\btheta$) with an additive $\epsilon$-relaxation: %
$
d_{\epsilon}(\bsigma^*,\bsigma) = \frac{1}{{n \choose 2}}\sum_{i < j} (g_{ij} + g_{ji})$, where each $g_{ij} = \1\left( (\theta_i > \theta_j + \epsilon) \wedge(\sigma(i) > \sigma(j)) \right)$. %
All reported performances are averaged across $50$ runs.

\textbf{Environments.}
We use four PL models: 1. {\it geo8} (with $n = 8$) 2. {\it arith10} (with $n = 10$) 3. {\it har20} (with $n = 20$) and 4. {\it arith50} (with $n = 50$). Their individual score parameters are as follows: \textbf{1. geo8:}
$
\theta_1=1$, and
$\frac{\theta_{i+1}}{\theta_{i}} = 0.875, ~\forall i \in [7]$. 
\textbf{2. arith10:} 
$\theta_1 = 1$ and $ \theta_{i} - \theta_{i+1} = 0.1, \, \forall i \in [9]$.
\textbf{3. har20:} $\theta = 1/(i), ~\forall i \in [20]$.
\textbf{4. arith50:} $\theta_1 = 1$ and $ \theta_{i} - \theta_{i+1} = 0.02, \, \forall i \in [9]$.

%
\begin{figure}[h!]
	\begin{center}
		\includegraphics[trim={0cm 0 0cm 0},clip,scale=0.5,width=0.45\textwidth]{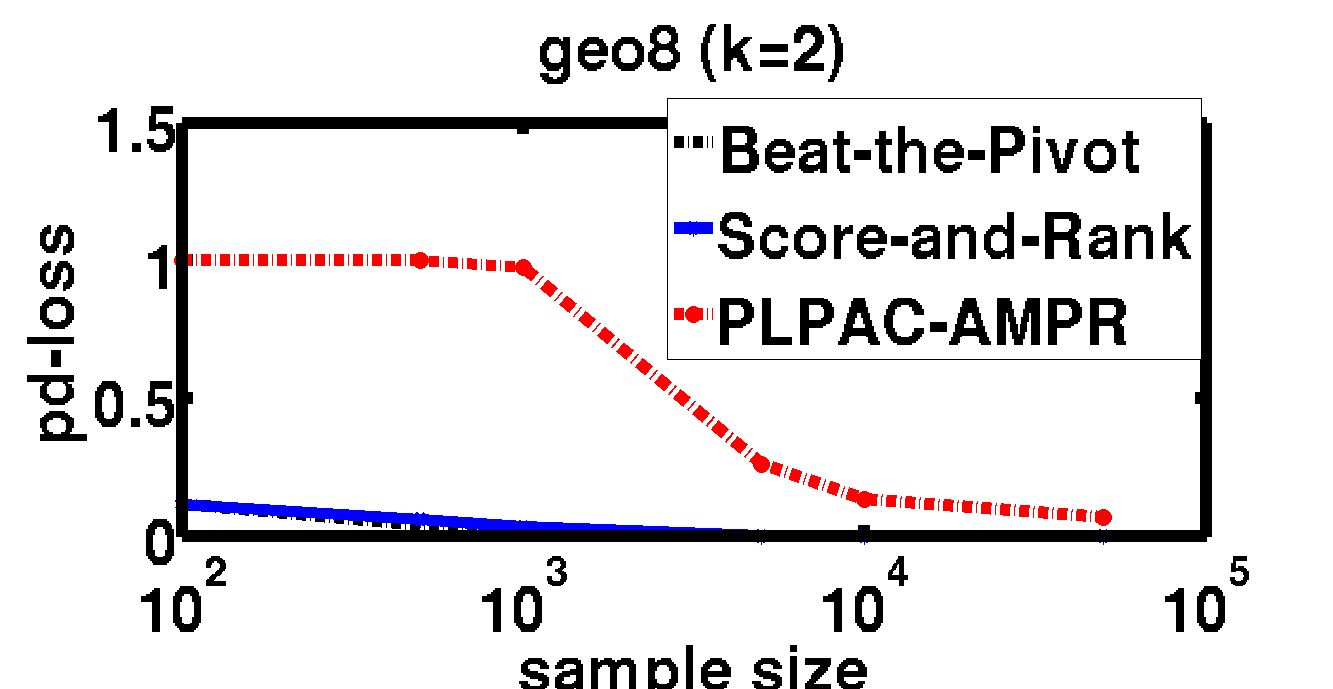}
		\hspace{0pt}
		\includegraphics[trim={0.cm 0 0cm 0},clip,scale=0.5,width=0.45\textwidth]{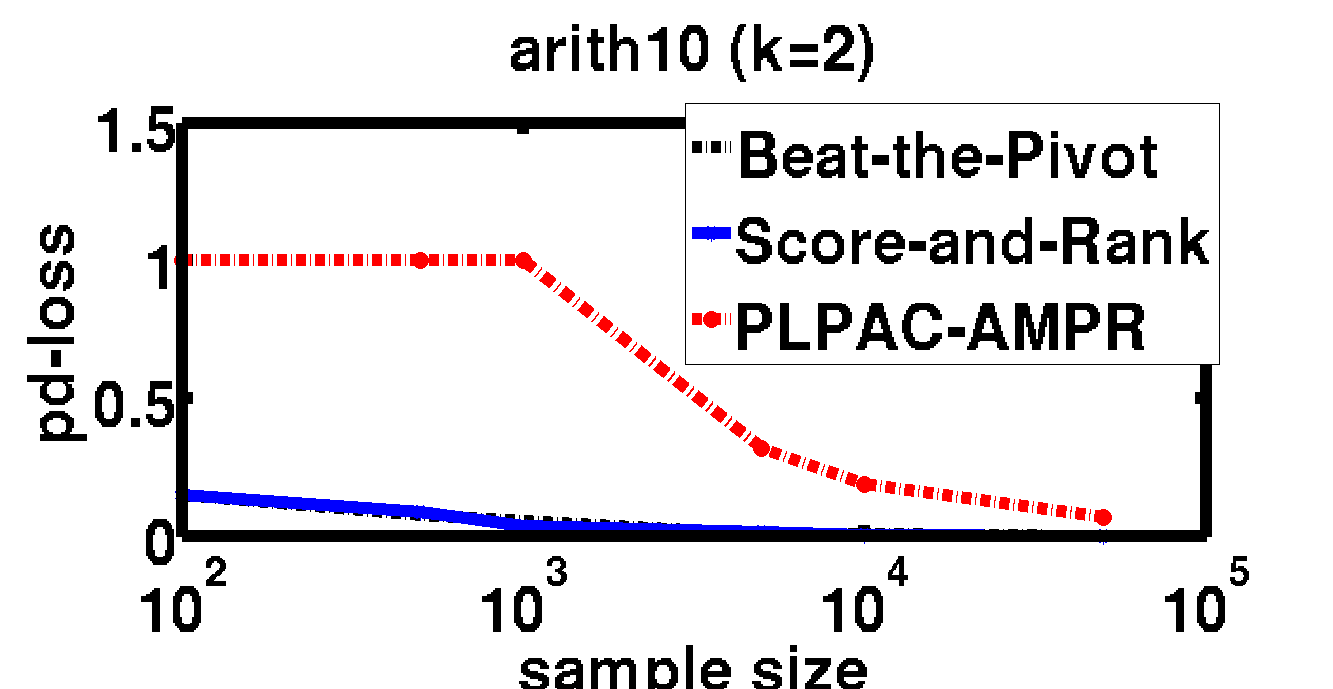}
		\hspace{0pt}
		\includegraphics[trim={0cm 0 0cm 0},clip,scale=0.5,width=0.45\textwidth]{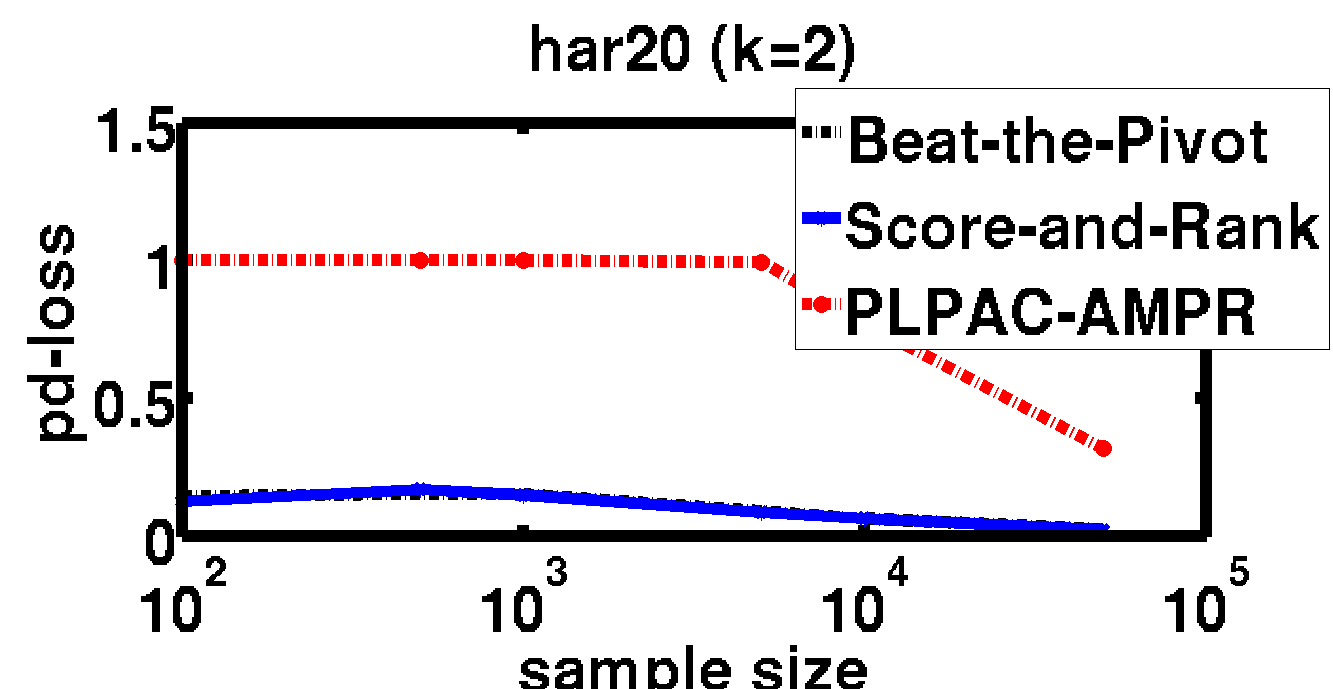}
		\hspace{0pt}
		\includegraphics[trim={0cm 0 0cm 0},clip,scale=0.5,width=0.45\textwidth]{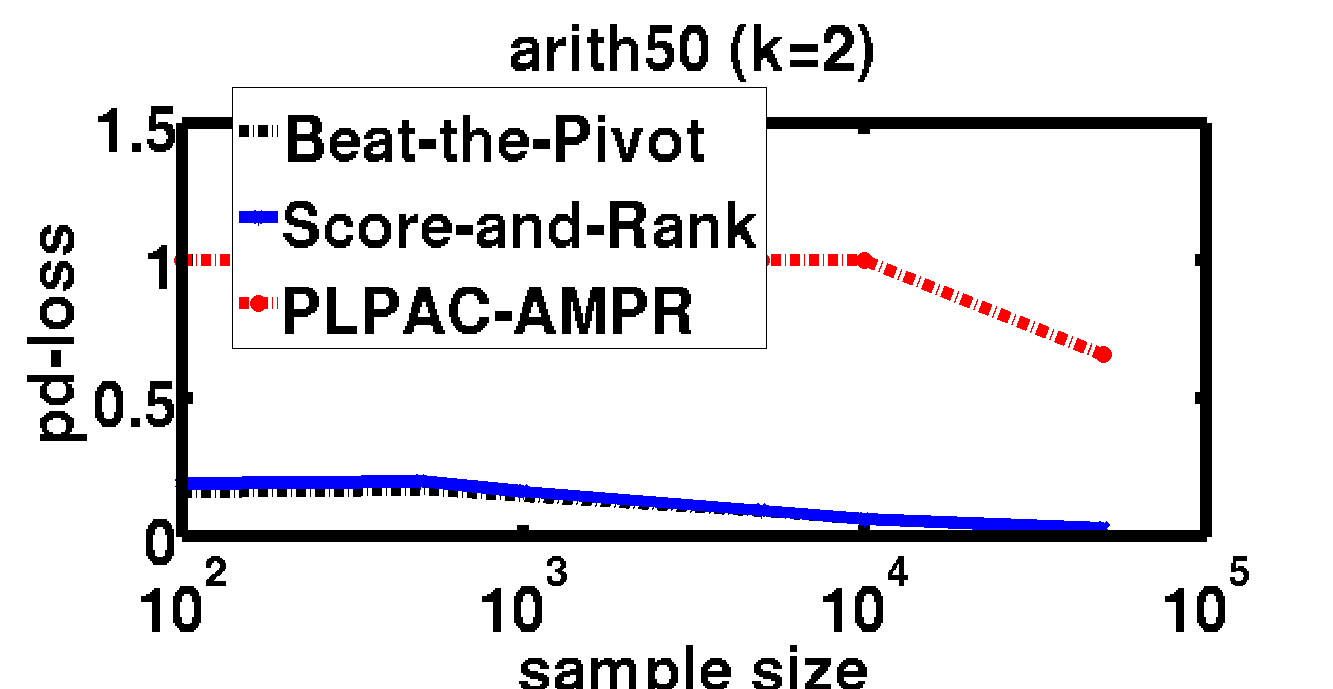}
		\vspace{-5pt}
		\caption{Ranking performance vs. sample size (\# rounds) with dueling plays ($k = 2$)}
		\label{fig:db}
		\vspace{-10pt}
	\end{center}
\end{figure}

\textbf{Ranking with Pairwise-Preferences ($k = 2$).} We first compare the above three algorithms with pairwise preference feedback, i.e. with $k = 2$ and $m = 1$ (WI feedback model). We set $\epsilon = 0.01$ and $\delta = 0.1$. Figure \ref{fig:db} clearly shows superiority of our two proposed algorithms over \emph{PLPAC-AMPR} \cite{Busa_pl} as they give much higher ranking accuracy given the sample size, rightfully justifying our improved theoretical guarantees as well (Theorem \ref{thm:batt_giant} and \ref{thm:est_thet}). Note that \emph{geo8} and \emph{arith50} are the easiest and hardest PL model instances, respectively; the latter has the largest $n$ with gaps $\theta_i - \theta_{i+1} = 0.02$. This also reflects in our experimental results as the ranking estimation loss being the highest for \emph{arith50} for all the algorithms, specifically  \emph{PLPAC-AMPR} very poorly till $10^4$ samples.

%
\vspace*{-10pt}
\begin{figure}[H]
	\begin{center}
		\includegraphics[trim={0cm 0 0cm 0},clip,scale=0.5,width=0.45\textwidth]{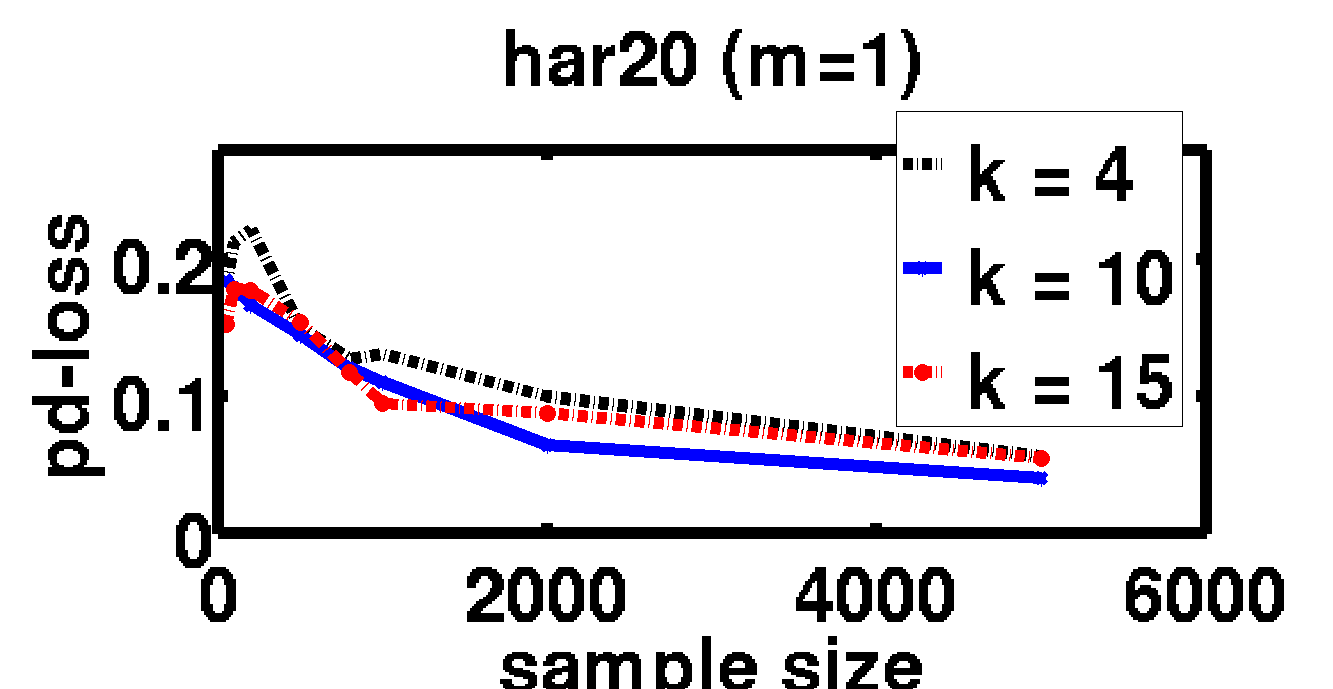}
		\hspace{0pt}
		\includegraphics[trim={0.cm 0 0cm 0},clip,scale=0.5,width=0.45\textwidth]{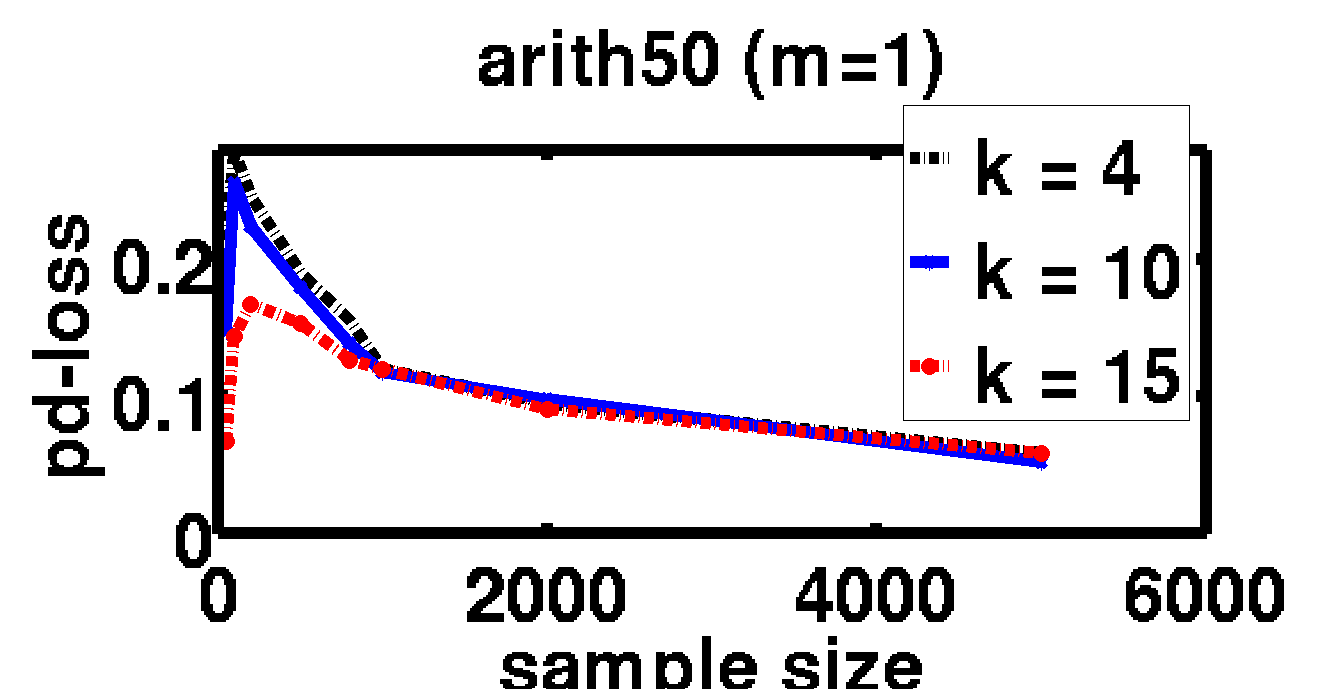}
		\vspace{-5pt}
		\caption{Ranking performance vs. subset size ($k$) with WI feedback ($m = 1$)} 
		\label{fig:varyk}
		\vspace{-10pt}
	\end{center}
\end{figure}
\vspace{-10pt}

\textbf{Ranking with Subsetwise-Preferences ($k > 2$ (with winner information (WI) feedback).} We next move to the setup of general subsetwise preference feedback $(k \ge 2)$ for WI feedback model (i.e. for $m = 1$) \footnote{\emph{PLPAC-AMPR} only works for $k = 2$ and is no longer applicable henceforth.}. We fix $\epsilon = 0.01$ and $\delta = 0.1$ and report the performance of \algant\,  on the datasets \emph{har20} and \emph{arith50}, varying $k$ over the range $4$ - $40$. As expected from Theorem \ref{thm:batt_giant} and explained in Remark \ref{rem:wi_null_k}, the ranking performance indeed does not seem to be varying with increasing subsetsize $k$ for WI feedback model for both PL models (Figure \ref{fig:varyk}).

%
\begin{figure}[H]
	\begin{center}
		\includegraphics[trim={0cm 0 0cm 0},clip,scale=1,width=0.45\textwidth]{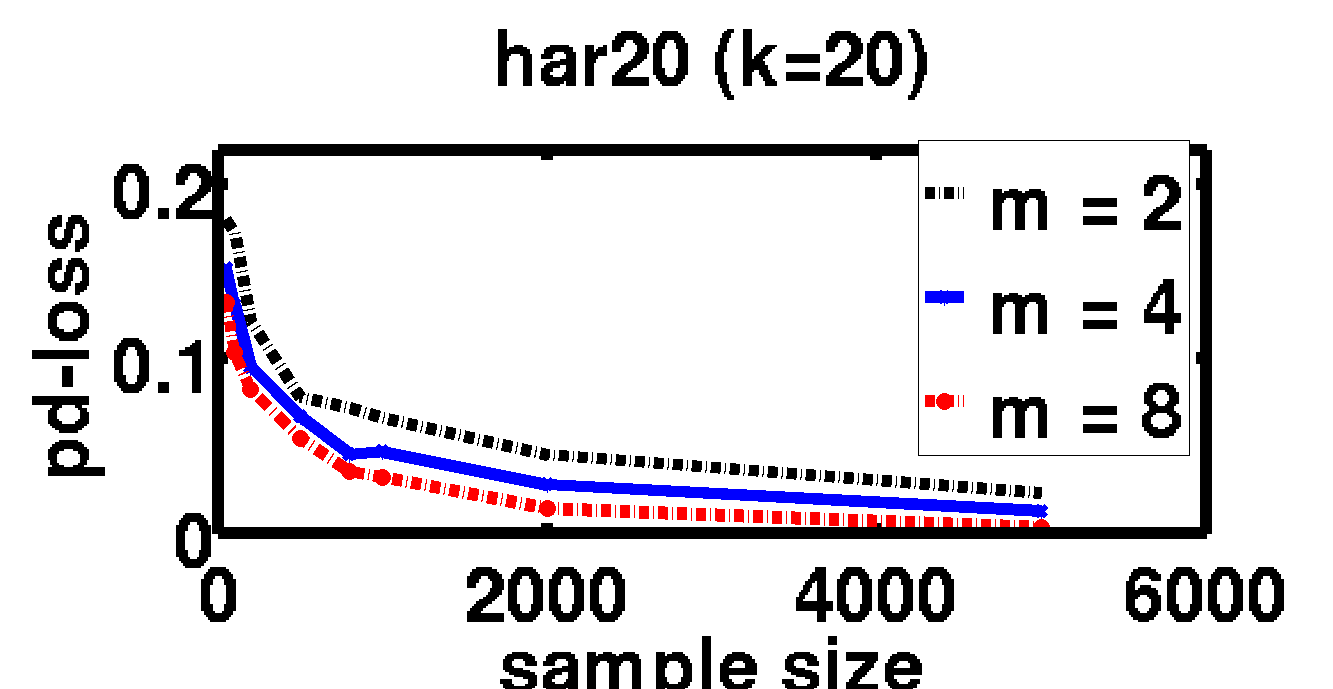}
		\hspace{0pt}
		\includegraphics[trim={0.cm 0 0cm 0},clip,scale=1,width=0.45\textwidth]{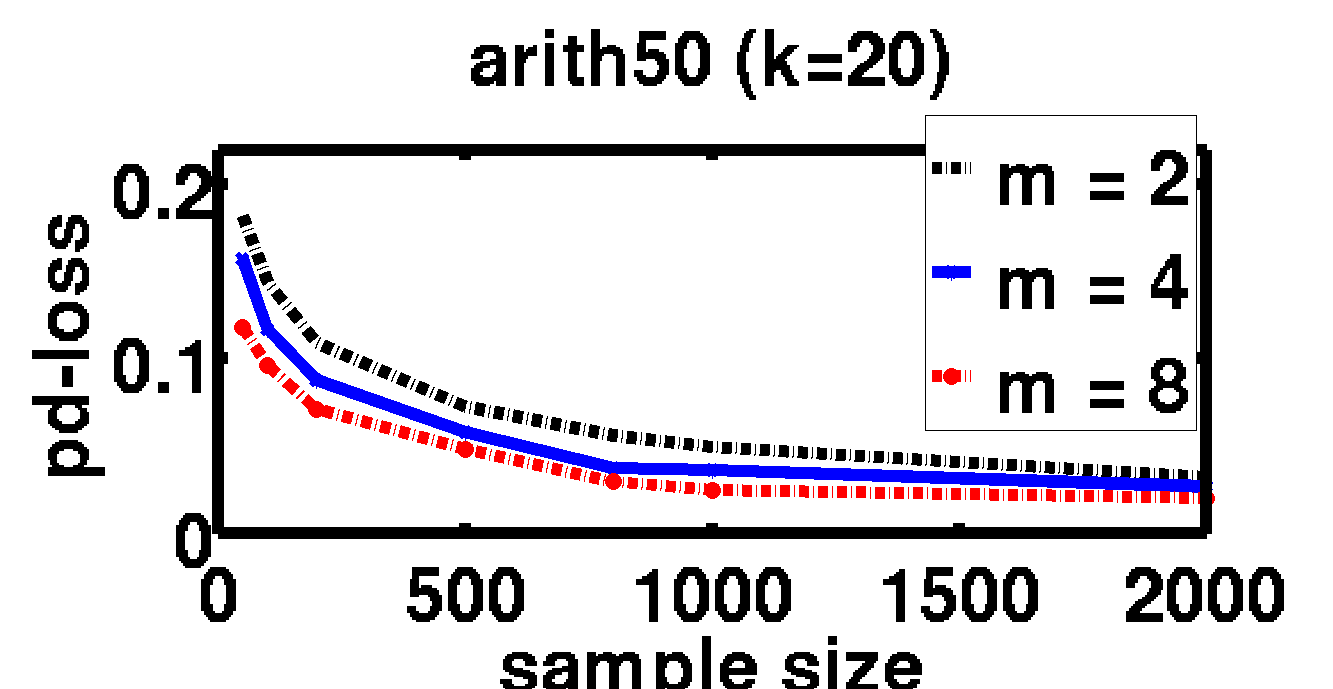}
		\vspace{-5pt}
		\caption{Ranking performance vs. feedback size ($m$) for fixed subset size ($k$)}
		\label{fig:varym}
		\vspace{-10pt}
	\end{center}
\end{figure}

	\vspace{-10pt}

\textbf{Ranking with Subsetwise-Preferences ($k > 2$ (with top-$m$ ranking (TR) feedback).} Lastly we report the performance of \algant\, for top-$m$ ranking (TR) feedback model (Algorithm \ref{alg:alg_ant_mod}) on two PL models: \emph{har20} (for $k = 20$) and \emph{arith50} (for $k = 45$), varying the range of $m$ from $2$ to $40$ (Figure \ref{fig:varym}). We set $\epsilon = 0.01$ and $\delta = 0.1$ as before. As expected, in this case it indeed reflects that the the ranking accuracy improves for larger $m$ given a fixed sample size---this reflects over theoretical guarantee of $\frac{1}{m}$-factor improvement of the sample complexity guarantee for TR feedback model (see Theorem \ref{thm:lb_pacpl_rnk} and Remark \ref{rem:1overm}).



\vspace*{-10pt}
\section{Conclusion and Future Work}
\label{sec:conclusion}
\vspace*{-10pt}




We have considered the PAC version of the problem of adaptively ranking $n$ items from $k$-subset-wise comparisons, 
in the Plackett-Luce (PL) preference model with winner information {(WI)} and top ranking {(TR)} feedback. 
With just WI, the required sample complexity lower bound is $\Omega\Big( \frac{n}{\epsilon^2} \ln \frac{n}{\delta} \Big)$, which is surprisingly independent of the subset size $k$. We have also designed two algorithms enjoying optimal sample complexity guarantees, and based on a novel \emph{pivoting-trick}. 
With {TR} feedback, a $\frac{1}{m}$-times faster learning rate is achievable, and we have given an algorithm with optimal sample complexity guarantees.

In the future, it would be of interest to analyse the problem with other choice models (e.g. multinomial probit, Mallows, nested logit, generalized extreme-value models, etc.), and perhaps to extend this theory to newer formulations such as assortment selection \cite{assort-discrete,assort-mallows}, revenue maximization with item prices \cite{assort-mnl,Agrawal+16}, or even in contextual scenarios \cite{CDB} where every individual user comes with their own model parameter.



\newpage

\bibliographystyle{plainnat}
\bibliography{bbpl-fr}

\newpage

\newpage
\onecolumn
\allowdisplaybreaks

\appendix
{
\section*{\centering \Large{Supplementary for Active Ranking with Subset-wise Preferences}}
}
  
\section{Appendix for Section \ref{sec:prelims}}  

\subsection{Proof of Lemma \ref{lem:pacobj_eqv}}  

\paceqv*

Recall that an algorithm is defined to be \pac\, (or \pac2\,) if it returns an \ebr\, (\ebrm\, ) with probability $(1-\delta)$.

\begin{proof}
\textbf{Case 1.} Suppose the algorithm is \pac. So if $\bsigma$ is the ranking returned by it, with high probability $(1-\delta)$, $\nexists$ two items $i,j \in [n]$ such that $\sigma(i) > \sigma(j)$ but $\theta_i - \theta_j \ge \epsilon $. But then this implies, $\nexists$ two items $i,j \in [n]$ with $\sigma(i) > \sigma(j)$ such that

\begin{align*}
Pr(i|\{i,j\}) - \frac{1}{2} = \frac{\theta_i - \theta_j}{2(\theta_i - \theta_j)} \ge \frac{\theta_i - \theta_j}{4b} = \frac{\epsilon}{4b} \ge \epsilon',
\end{align*}

which proves our first claim.

\textbf{Case 2.} Now suppose the algorithm is \pacdb. So if $\bsigma$ is the ranking returned by it, with high probability $(1-\delta)$, $\nexists$ two items $i,j \in [n]$ such that $\sigma(i) > \sigma(j)$ but $Pr(i|\{i,j\}) - \frac{1}{2} \ge \epsilon $. But since $Pr(i|\{i,j\}) = \frac{\theta_i}{\theta_i+\theta_j}$, this them equivalently implies, $\nexists$ two items $i,j \in [n]$ with $\sigma(i) > \sigma(j)$ such that

\begin{align*}
& \frac{\theta_i}{\theta_j} \ge \frac{1/2 + \epsilon}{1/2 - \epsilon}\\
& \implies \theta_i  \ge \theta_j\Bigg(\frac{1/2 + \epsilon}{1/2 - \epsilon}\Bigg) \ge \theta_j\Big({\frac{1}{2} + \epsilon}\Big)^2\\
& \implies \theta_i - \theta_j \ge \theta_j\big(4\epsilon^2 + 4\epsilon\big) \ge 4a\epsilon(1+\epsilon) \ge \epsilon',
\end{align*}
which proves our second claim and concludes the proof.

\end{proof}

\subsection{Proof of Lemma \ref{lem:pl_simulator}}  

\plsimulator*

\begin{proof}
	We prove the lemma by using a coupling argument. Consider the following `simulator' or probability space for the Plackett-Luce choice model that specifically depends on the item pair $i,j$, constructed as follows. Let $Z_1, Z_2, \ldots$ be a sequence of iid Bernoulli random variables with success parameter $\theta_i/(\theta_i + \theta_j)$. A counter is first initialized to $0$. At each time $t$, given $S_1, i_1, \ldots, S_{t-1}, i_{t-1}$ and $S_t$,  an independent coin is tossed with probability of heads $(\theta_i + \theta_j)/\sum_{k \in S_t} \theta_k$. If the coin lands tails, then $i_t$ is drawn as an independent sample from the Plackett-Luce distribution over $S_t \setminus \{i,j\}$, else, the counter is incremented by $1$, and $i_t$ is returned as $i$ if $Z(C) = 1$ or $j$ if $Z(C) = 0$ where $C$ is the present value of the counter.
	
	It may be checked that the construction above indeed yields the correct joint distribution for the sequence $i_1, S_1, \ldots, i_T, S_T$ as desired, due to the independence of irrelevant alternatives (IIA) property of the Plackett-Luce choice model: 
	\[ Pr(i_t = i | i_t \in \{i,j\}, S_t) = \frac{Pr(i_t = i | S_t)}{Pr(i_t \in \{i,j\} | S_t)} = \frac{\theta_i/\sum_{k \in S_t} \theta_k}{(\theta_i + \theta_j)/\sum_{k \in S_t} \theta_k} = \frac{\theta_i}{\theta_i + \theta_j}. \]
	Furthermore, $i_t \in \{i,j\}$ if and only if $C$ is incremented at round $t$, and $i_t = i$ if and only if $C$ is incremented at round $t$ and $Z(C) = 1$. We thus have
\begin{align*}
	Pr &\left( \frac{n_i(T)}{n_{ij}(T)} - \frac{\theta_i}{\theta_i + \theta_j} \ge \eta, \; n_{ij}(T) \geq v \right) = Pr\left(\frac{\sum_{\ell=1}^{n_{ij}(T)} Z_\ell}{n_{ij}(T)} - \frac{\theta_i}{\theta_i + \theta_j}\ge\eta, \; n_{ij}(T) \geq v \right) \\
	&= \sum_{m = v}^T Pr\left(\frac{\sum_{\ell=1}^{n_{ij}(T)} Z_\ell}{n_{ij}(T)} - \frac{\theta_i}{\theta_i + \theta_j}\ge\eta, \; n_{ij}(T) = m \right) \\
	&= \sum_{m = v}^T Pr\left(\frac{\sum_{\ell=1}^{m} Z_\ell}{m} - \frac{\theta_i}{\theta_i + \theta_j}\ge\eta, \; n_{ij}(T) = m \right) \\
	&\stackrel{(a)}{=} \sum_{m = v}^T Pr\left(\frac{\sum_{\ell=1}^{m} Z_\ell}{m} - \frac{\theta_i}{\theta_i + \theta_j}\ge\eta \right) \, Pr \left( n_{ij}(T) = m \right) \\
	&\stackrel{(b)}{\leq} \sum_{m = v}^T Pr \left( n_{ij}(T) = m \right) \, e^{-2m \eta^2} \leq e^{-2v \eta^2},
\end{align*}	
where $(a)$ uses the fact that $S_1, \dots, S_T, X_1, \ldots, X_T$ are independent of $Z_1, Z_2, \ldots,$, and so $n_{ij}(T) \in \sigma(S_1, \dots, S_T, X_1, \ldots, X_T)$ is independent of $Z_1, \ldots, Z_m$ for any fixed $m$, and $(b)$ uses Hoeffding's concentration inequality for the iid sequence $Z_i$. 

Similarly, one can also derive
\[
Pr\left( \frac{n_i(T)}{n_{ij}(T)} - \frac{\theta_i}{\theta_i + \theta_j} \le -\eta, \; n_{ij}(T) \geq v \right) \le e^{-2v \eta^2}, 
\]
which concludes the proof.
\end{proof}

\section{Appendix for Section \ref{sec:est_pl_score}}  


\subsection{Proof of Lemma \ref{lem:geo2mgf}}
\label{app:geo2mgf}

\geomgf*

\begin{proof}
$(1)$ follows from the simple observation probability of item $b$ winning at any trial $t$ is independent and identically distributed (\emph{iid}) $Ber\Big( \frac{\theta_b}{\sum_{j \in S}} \Big)$ and $T$ essentially denotes the number of trials till first success (win of $b$). (Recall from Section \ref{sec:prelims} that $Geo(p)$ denote the `number of trials before success' version of the Geometric random variable with probability of success at each trial being $p \in [0,1]$).

We proof $(2)$ by deriving the moment generating function (MGF) of the random variable $w_i(T)$ which gives:

\begin{restatable}[MGF of $w_i(T)$]{lem}{mgf}
\label{lem:mgf}
For any item $i \in S\setminus\{b\}$, the moment generating function of the random variable $w_i(T)$ is given by: 
$
\E\Big[ e^{\lambda w_i(T)} \Big] = \frac{1}{1-\frac{\theta_i}{\theta_b}(e^{\lambda - 1})}, \forall \ell \in [w_1], \, \text{ for any } \lambda \in \big(0,\ln(1+\eta)\big), \text{ with } \eta < \min_{j \in S}\frac{\theta_b}{\theta_j}
$.
\end{restatable}

See Appendix \ref{app:mgf} for the proof. Now firstly recall that the MGF of any random variable $X \sim Geo(p)$ is given by $\E[e^{\lambda X}] = \frac{p}{(1-e^{\lambda }(1-p))}, \, \forall \lambda \in \Big(0, -\ln (1-p) \Big)$. In the current case $p = \frac{\theta_b}{\theta_b+\theta_i}$. Thus we have $\Bigg(\frac{1}{p} - 1\Bigg) = \frac{\theta_i}{\theta_b}$ or $\frac{1}{1-p} = \frac{\theta_i + \theta_b}{\theta_i}$, and the MGF holds good for any $\lambda \in (0,\ln(1+\eta))$ as long as $\eta < \min_{j\in S}\frac{\theta_b}{\theta_j}$.

The proof now follows from straightforward reduction of Lemma \ref{lem:mgf}. Formally,  we have:

\begin{align*}
\E \Big[ e^{\lambda X} \Big] & =  \frac{p}{(1-e^{\lambda }(1-p))} ~~\text{ for any } \lambda \in \big(0,\ln(1+\eta)\big), \text{ where } \eta < \min_{j \in S}\frac{\theta_b}{\theta_j}, \\
& = \frac{1}{\frac{1}{p}-e^{\lambda}\Big(\frac{1-p}{p}\Big)} = \frac{1}{1+ \frac{\theta_i}{\theta_b} - e^\lambda\big(\frac{\theta_i}{\theta_b}\big)}\\
& = \frac{1}{1-\frac{\theta_i}{\theta_b}(e^{\lambda - 1})} = \E\Big[ e^{\lambda w_i(T)} \Big], 
\end{align*} 

where the last equality follows from Lemma \ref{lem:mgf} as two random variables with same MGF must have same distributions. This concludes the proof.
\end{proof}

\subsection{Proof of Lemma \ref{lem:mgf}} 
\label{app:mgf}
 
\mgf*

\begin{proof}
The proof follows from using standard MGF results of Bernoulli and Geometric random variables. We denote $S_{-b} = S \sm \{b\}$, $\hat T = (T-1)$, $p = \frac{\theta_b}{\sum_{j \in S}\theta_j}$, and $p' = \frac{\theta_i}{\sum_{j \in S_{-b}}\theta_j}$. As argued in Lemma \ref{lem:geo2mgf}, we know that ${\hat T} \sim$ Geo$\Big(p\Big)$. Also given a fixed (non-random) ${\hat T}$, $w_i(T) \sim$ Bin$\Big({\hat T},p'\Big)$. Then using law of iterated expectation:

\begin{align*}
\E\Big[ e^{\lambda w_i(T)} \Big] & = \E_{{\hat T}}\Big[ \E\big[ e^{\lambda w_i(T)} \mid {\hat T} \big] \Big]\\
& = \E_{{\hat T}}\Big[ (p'e^\lambda + 1 - p')^{{\hat T}} \Big],\\
\end{align*}
where the last equality follows from the MGF of Binomial random variables. Note that, since $\lambda > 0$, we have $(p'e^\lambda + 1 - p')  = 1 + p'(e^\lambda -1) > 1$. Let us denote $\lambda' = \ln(1 + p'(e^\lambda -1))$. Clearly $ \lambda' > 0$ as both $\lambda,p' > 0$. 
Then from above equation, one can write:

\begin{align*}
\E\Big[ e^{\lambda w_i(T)} \Big] & = \E_{{\hat T}}\Big[ e^{\lambda'{\hat T}} \Big],\\
& = \frac{p}{(1-e^{\lambda'}(1-p))}\\
& = \frac{p}{1-(1 + p'(e^\lambda -1))(1-p)} = \frac{1}{1-\frac{\theta_i}{\theta_b}(e^{\lambda - 1})},
\end{align*}

where the second equality follows from the result that MGF of a geometric random variable $X \sim$ Geo$(p)$ is: $\E[e^{\lambda'X}] = \frac{p}{(1-e^{\lambda'}(1-p))}, \, \forall \lambda' \in \Big(0, -\ln (1-p) \Big)$. 

Thus the only remaining thing to show is $\lambda'$ indeed satisfies the above range. As argues above, clearly $ \lambda' > 0$ as both $\lambda,p' > 0$. 
To verify the upper bound, note that by choice $\lambda < \ln \Big( 1+\frac{\theta_b}{\theta_j}\Big), \, \forall j \in S$, which implies $e^\lambda  < (1+ \frac{\theta_b}{\theta_i})$, for any $i \in S_{-b}$. This further implies $(e^\lambda - 1)\frac{\theta_i}{\theta_b} < 1 \implies (1 - \frac{\theta_i}{\theta_b}(e^\lambda -1)) > 0 \implies (1-p)(1+p'(e^\lambda - 1)) < 1$ rearranging which leads to the desired bound $\lambda' < -\ln (1-p)$ (recall $\lambda' = \ln(1+p'(e^\lambda-1))$,
and thus the above MGF holds good. This concludes the proof.
\end{proof}

\subsection{Proof of Lemma \ref{lem:geoconc}}

\geoconc*

\begin{proof}
The result follows from the concentration of Geometric random variable as shown in \cite{geoconc}. Note that, $Z$ denotes the number of trials needed to get $n$ wins of item $b$, where the probability of success (i.e. item $b$ winning) at each trial is $\frac{\theta_b}{\theta_b + \theta_i}$. Thus $Z \sim $ \emph{NB}$(n,\frac{\theta_b}{\theta_b + \theta_i})$. Clearly, by applying union bounding we get:

\begin{align*}
Pr\Big( \Big|\frac{Z}{n} - \frac{\theta_i}{\theta_b}\Big| > \eta \Big) \le Pr\Big( \frac{Z}{n} - \frac{\theta_i}{\theta_b} > \eta \Big) + Pr\Big( \frac{Z}{n} - \frac{\theta_i}{\theta_b} < -\eta \Big).
\end{align*}

Let us start by analysing the first term $Pr\Big( \frac{Z}{n} - \frac{\theta_i}{\theta_b} > \eta \Big)$.

\begin{align}
\label{eq:conc1}
\nonumber Pr& \Big( \frac{Z}{n} - \frac{\theta_i}{\theta_b} > \eta \Big) =  Pr\Big( Z > n\frac{\theta_i}{\theta_b} + n\eta \Big)\\
\nonumber & \le Pr \Bigg( Bin\bigg(n(\frac{\theta_i}{\theta_b} + 1) +  n\eta, \frac{\theta_b}{\theta_b + \theta_i} \bigg) < n \Bigg)\\
\nonumber & \le Pr \Bigg( Bin\bigg(n(\frac{\theta_i}{\theta_b} + 1) +  n\eta, \frac{\theta_b}{\theta_b + \theta_i} \bigg) - \bigg[ n(\frac{\theta_i}{\theta_b} + 1) +  n\eta\bigg]\frac{1}{1 + \frac{\theta_i}{\theta_b}}  < - \frac{n \eta}{1 + \frac{\theta_i}{\theta_b}} \Bigg)\\
& \le \exp\Big(-\frac{2}{m}\tilde \eta^2\Big) = \exp\Bigg( - \frac{2n\eta^2}{\Big( 1+\frac{\theta_i}{\theta_b} \Big)^2\Big(\eta + 1+\frac{\theta_i}{\theta_b}  \Big)}\Bigg)
\end{align}
where the last inequality follows simply apply Hoeffding's inequality with $m = n(\frac{\theta_i}{\theta_b} + 1) +  n\eta$ and $\tilde \eta = \frac{n \eta}{1 + \frac{\theta_i}{\theta_b}}$. Using a similar derivation as before, one can also show:

\begin{align}
\label{eq:conc2}
Pr\Big( \frac{Z}{n} - \frac{\theta_i}{\theta_b} < -\eta \Big) \le \exp\Bigg( - \frac{2n\eta^2}{\Big( 1+\frac{\theta_i}{\theta_b} \Big)^2\Big(\eta + 1+\frac{\theta_i}{\theta_b}  \Big)}\Bigg)
\end{align}

The result now follows combining \eqref{eq:conc1} and \eqref{eq:conc2}.
\end{proof}

\newpage

\section{Appendix for Section \ref{sec:algo_wi}}


\subsection{Proof of Lemma \ref{lem:trace_best}}
\label{app:trace_best}


\ubtrc*

\begin{proof}
We start by analyzing the required sample complexity first. Note that the `while loop' of Algorithm \ref{alg_bi} always discards away $k-1$ items per iteration. Thus, $n$ being the total number of items the loop can be executed is at most for $\lceil \frac{n}{k-1} \rceil $ many number of iterations. Clearly, the sample complexity of each iteration being $t = \frac{2k}{\epsilon^2}\ln \frac{2\delta}{n}$, the total sample complexity of the algorithm becomes $\big( \lceil \frac{n}{k-1} \rceil \big)\frac{2k}{\epsilon^2}\ln \frac{n}{2\delta} \le \big(  \frac{n}{k-1} + 1 \big)\frac{2k}{\epsilon^2}\ln \frac{n}{2\delta} = \big( n+ \frac{n}{k-1} + k \big)\frac{2}{\epsilon^2}\ln \frac{n}{2\delta} = O(\frac{n}{\epsilon^2}\ln \frac{n}{\delta})$.

We now prove the $(\epsilon,\delta)$-{PAC} correctness of the algorithm. As argued before, the `while loop' of Algorithm \ref{alg_bi} can run for maximum $\lceil \frac{n}{k-1} \rceil $ many number of iterations. We denote the iterations by $\ell = 1,2, \ldots \lceil \frac{n}{k-1} \rceil $, and the corresponding set $\cA$ of iteration $\ell$ by $\cA_\ell$. 

Note that our idea is to retain the estimated best item in `running winner' $r_\ell$ and compare it with the `empirical best item' $c_\ell$ of $\cA_\ell$ at every iteration $\ell$. The crucial observation lies in noting that at any iteration $\ell$, $r_\ell$ gets updated as follows:

\begin{restatable}[]{lem}{lemtrc}
\label{lem:c_vs_r}
At any iteration $\ell = 1,2 \ldots \big \lfloor \frac{n}{k-1} \big \rfloor$, with probability at least $(1-\frac{\delta}{2n})$, Algorithm \ref{alg_bi} retains $r_{\ell+1} \leftarrow r_\ell$ if $p_{c_\ell r_\ell } \le \frac{1}{2}$, and sets $r_{\ell+1} \leftarrow c_\ell$ if $p_{c_\ell r_\ell} \ge \frac{1}{2} + \epsilon$.
\end{restatable}

\begin{proof}
Consider any set $\cA_\ell$, by which we mean the state of $\cA$ in the algorithm at iteration $\ell$. The crucial observation to make is that since $c_\ell$ is the empirical winner of $t$ rounds, then $w_{c_\ell} \ge \frac{t}{k}$. Thus $w_{c_\ell} + w_{r_\ell} \ge \frac{t}{k}$. Let $n_{ij}:= w_i + w_j$ denotes the total number of pairwise comparisons between item $i$ and $j$ in $t$ rounds, for any $i,j \in \cA_\ell$. Then clearly, $0 \le n_{ij} \le t$ and $n_{ij} = n_{ji}$. Specifically we have $\hp_{r_\ell c_\ell} = \frac{w_{r_\ell}}{w_{r_\ell}+w_{c_\ell}} = \frac{w_{r_\ell}}{n_{r_\ell c_\ell}}$.
We prove the claim by analyzing the following cases: 

\textbf{Case 1.} (If $p_{c_\ell r_\ell} \le \frac{1}{2}$, \algwin\, retains $r_{\ell+1} \leftarrow r_\ell$): Note that \algwin\, replaces $r_{\ell+1}$ by $c_\ell$ only if $\hp_{c_\ell,r_\ell} > \frac{1}{2} + \frac{\epsilon}{2} $, but this happens with probability: 

\begin{align*}
& Pr\Bigg( \bigg\{ \hp_{c_\ell r_\ell} > \frac{1}{2} + \frac{\epsilon}{2} \bigg\} \Bigg) \\
& = Pr\Bigg( \bigg\{ \hp_{c_\ell r_\ell} > \frac{1}{2} + \frac{\epsilon}{2} \bigg\} \cap \bigg\{ n_{c_\ell r_\ell} \ge \frac{t}{k} \bigg\}\Bigg) + \cancelto{0}{Pr\bigg\{ n_{c_\ell r_\ell} < \frac{t}{k} \bigg\}}Pr\Bigg( \bigg\{ \hp_{c_\ell r_\ell} > \frac{1}{2} + \frac{\epsilon}{2} \bigg\} \Big | \bigg\{ n_{c_\ell r_\ell} < \frac{t}{k} \bigg\}\Bigg)\\
& \le Pr\Bigg( \bigg\{ \hp_{c_\ell r_\ell} - p_{c_\ell r_\ell} > \frac{\epsilon}{2} \bigg\} \cap \bigg\{ n_{c_\ell r_\ell} \ge \frac{t}{k} \bigg\} \Bigg) \le \exp\Big( -2\dfrac{t}{k}\bigg(\frac{\epsilon}{2}\bigg)^2 \Big) = \frac{\delta}{2n},
\end{align*}
where the first inequality follows as $p_{c_\ell r_\ell} \le \frac{1}{2}$, and the second inequality is by applying Lemma \ref{lem:pl_simulator} with $\eta = \frac{\epsilon}{2}$ and $v = \frac{t}{k}$.
We now proceed to the second case:

\textbf{Case 2.} (If $p_{c_\ell r_\ell} \ge \frac{1}{2} + \epsilon$, \algwin\,  sets $r_{\ell+1} \leftarrow c_\ell$): Recall again that \algwin\, retains $r_{\ell+1}  \leftarrow r_\ell$ only if $\hp_{c_\ell,r_\ell} \le \frac{1}{2} + \frac{\epsilon}{2} $. This happens with probability:

\begin{align*}
& Pr\Bigg( \bigg\{ \hp_{c_\ell r_\ell} \le \frac{1}{2} + \frac{\epsilon}{2} \bigg\} \Bigg) \\
& = Pr\Bigg( \bigg\{ \hp_{c_\ell r_\ell} \le \frac{1}{2} + \frac{\epsilon}{2} \bigg\} \cap \bigg\{ n_{c_\ell r_\ell} \ge \frac{t}{k} \bigg\}\Bigg) + \cancelto{0}{Pr\bigg\{ n_{c_\ell r_\ell} < \frac{t}{k} \bigg\}}Pr\Bigg( \bigg\{ \hp_{c_\ell r_\ell} \le \frac{1}{2} + \frac{\epsilon}{2} \bigg\} \Big | \bigg\{ n_{c_\ell r_\ell} < \frac{t}{k} \bigg\}\Bigg)\\
& = Pr\Bigg( \bigg\{ \hp_{c_\ell r_\ell} \le \frac{1}{2} + \epsilon - \frac{\epsilon}{2} \bigg\} \cap \bigg\{ n_{c_\ell r_\ell} \ge \frac{t}{k} \bigg\} \Bigg) \\
& \le Pr\Bigg( \bigg\{ \hp_{c_\ell r_\ell} - p_{c_\ell r_\ell} \le - \frac{\epsilon}{2} \bigg\} \cap \bigg\{ n_{c_\ell r_\ell} \ge \frac{t}{k} \bigg\} \Bigg) \le \exp\Big( -2\dfrac{t}{k}\bigg(\frac{\epsilon}{2}\bigg)^2 \Big) = \frac{\delta}{2n},
\end{align*}
where the first inequality holds as $p_{c_\ell r_\ell} \ge \frac{1}{2} + \epsilon$, and the second one by applying Lemma \ref{lem:pl_simulator} with $\eta = \frac{\epsilon}{2}$ and $v = \frac{t}{k}$. Combining the above two cases concludes the proof.
\end{proof}

Given Algorithm \ref{alg_bi} satisfies Lemma \ref{lem:c_vs_r}, and taking union bound over $(k-1)$ elements in $\cA_\ell \setminus\{r_\ell\}$, we get that with probability at least $(1-\frac{(k-1)\delta}{2n})$,

\begin{align}
\label{eq:r_vs_c}
p_{r_{\ell+1}r_\ell} \ge \frac{1}{2} \text{ and, } p_{r_{\ell+1}c_\ell} \ge \frac{1}{2} - \epsilon.
\end{align} 

Above suggests that for each iteration $\ell$, the estimated `best' item $r_\ell$ only gets improved as $p_{r_{\ell+1}r_\ell} \ge \frac{1}{2}$. Let, $\ell_*$ denotes the specific iteration such that $1 \in \cA_\ell$ for the first time, i.e. $\ell_* = \min\{ \ell \mid 1 \in \cA_\ell \}$. Clearly $\ell_* \le \lceil \frac{n}{k-1} \rceil$. 
Now \eqref{eq:r_vs_c} suggests that with probability at least $(1-\frac{(k-1)\delta}{2n})$, $p_{r_{\ell_*+1}1} \ge \frac{1}{2} - \epsilon$. Moreover \eqref{eq:r_vs_c} also suggests that for all $\ell > \ell_*$, with probability at least $(1-\frac{(k-1)\delta}{2n})$, $p_{r_{\ell+1}r_\ell} \ge \frac{1}{2}$, which implies for all $\ell > \ell_*$, $p_{r_{\ell+1}1} \ge \frac{1}{2} - \epsilon$ as well -- This holds due to the following transitivity property of the Plackett-Luce model: For any three items $i_1,i_2,i_3 \in [n]$, if $p_{i_1i_2}\ge \frac{1}{2}$ and $p_{i_2i_3}\ge \frac{1}{2}$, then we have $p_{i_1i_3}\ge \frac{1}{2}$ as well. 

This argument finally leads to $p_{r_*1} \ge \frac{1}{2} - \epsilon$. Since failure probability at each iteration $\ell$ is at most $\frac{(k-1)\delta}{2n}$, and Algorithm \ref{alg_bi} runs for maximum $\lceil \frac{n}{k-1} \rceil$ many number of iterations, using union bound over $\ell$, the total failure probability of the algorithm is at most $\lceil \frac{n}{k-1} \rceil \frac{(k-1)\delta}{2n} \le (\frac{n}{k-1}+1)\frac{(k-1)\delta}{2n} = \delta\Big( \frac{n+k-1}{2n} \Big) \le \delta$ (since $k \le n$). This concludes the correctness of the algorithm showing that it indeed satisfies the $(\epsilon,\delta)$-PAC objective.
\end{proof}

\subsection{Proof of Theorem \ref{thm:batt_giant}}

\ubpiv*

\begin{proof}
We first analyze the sample complexity of \algant. Clearly, there are at most $G = \lceil \frac{n-1}{k-1} \rceil \le \frac{n-1}{k-1} + 1 \le \frac{2(n-1)}{k-1}$ groups. Here the last inequality follows since $n \ge k$. 
Now each group $\cG_g$ (set of $k$ items) is played (queried) for at most $t:= \frac{2k}{\epsilon'^2}\log \frac{1}{\delta'}$ times, which gives the total sample complexity of the algorithm to be 
\[
G*t \le  \frac{2(n-1)}{(k-1)}*\frac{2k}{\epsilon'^2}\log \frac{1}{\delta'} = \frac{2048(n-1)}{\epsilon^2}\log \frac{8n}{\delta} = O\Big( \frac{n}{\epsilon^2}\log \frac{n}{\delta} \Big),
\]
where we used the fact $k \ge 2$ and bound $\frac{k}{k-1} \le 2$.

Moreover the sample complexity of \algwin\, is also $O\Big( \frac{n}{\epsilon^2}\log \frac{n}{\delta}\Big)$ as proved in Lemma \ref{lem:trace_best}. Combining this with above thus makes the total sample complexity of \algant \, $O\Big( \frac{n}{\epsilon^2}\log \frac{n}{\delta}\Big)$.

We are now only left to show the correctness of the algorithm, i.e. \algant\, indeed \pac\, in the above sample complexity.
We start by proving the following lemma which would be crucial throughout the analysis. Let us first denote $\Delta^b_{ij} = P(i \succ b) - P(j \succ b), \,$ for any $i$ and $j \in [n]$.

\begin{lem}
\label{lem:p_1i}
If $b$ is the pivot-item returned by Algorithm \ref{alg_bi} (Line $5$ of \algant), then
for any two items $i,j \in [n]$, such that $\theta_i \ge \theta_j$, $\frac{(\theta_i - \theta_j)}{8} \le \Delta^b_{ij} \le 4(\theta_i - \theta_j)$, with probability at least $\big( 1 - \frac{\delta}{2}\big)$.
\end{lem}

\begin{proof}
First let us assume if $b = 1$.
Then $\Delta^b_{ij} = P(i \succ 1) - P(j \succ 1) = \frac{\theta_1(\theta_i - \theta_j)}{(\theta_j + \theta_1)(\theta_i + \theta_1)} \ge \frac{(\theta_i - \theta_j)}{4}$, since $\theta_1 = 1$ and $\theta_i \le 1, ~\forall i \in [n]$.
On the other hand, we also have $\Delta^b_{ij} = P(i \succ 1) - P(j \succ 1) = \frac{\theta_1(\theta_i - \theta_j)}{(\theta_j + \theta_1)(\theta_i + \theta_1)} \le (\theta_i - \theta_j) = {\epsilon}$, since $\theta_1 = 1$ and $\theta_i \ge 0, \forall i \in [n]$. 

However even if $b \neq 1$, with high probability $(1-\frac{\delta}{2})$, it is ensured that $\theta_b > \theta_1 - \frac{1}{2}$, as with high probability $(1-\frac{\delta}{2})$, Algorithm \ref{alg_bi} returns an $\epsilon$-\emph{Best-Item} (see Lemma \ref{lem:trace_best}, proof in Appendix \ref{app:trace_best}). Then similarly as before, $\Delta^b_{ij} = P(i \succ 1) - P(j \succ 1) = \frac{\theta_b(\theta_i - \theta_j)}{(\theta_j + \theta_b)(\theta_i + \theta_b)} \ge \frac{(\theta_i - \theta_j)}{8}$, $\theta_b \ge \theta_1 - \frac{1}{2} = \frac{1}{2}$, and $\theta_i \le 1,~ \forall i \in [n]$.
On the other hand, we also have $\Delta^b_{ij} = P(i \succ b) - P(j \succ b) = \frac{\theta_b(\theta_i - \theta_j)}{(\theta_j + \theta_b)(\theta_i + \theta_b)} \le 4(\theta_i - \theta_j) = {\epsilon}$, since $\theta_b \in \big(\frac{1}{2},1]$, and $\theta_i \ge 0, ~\forall i \in [n]$. This proves our claim.
\end{proof}

Now to ensure the correctness of \algant, recall that all we need to show it returns an \ebr\ $\bsigma \in \Sigma_{[n]}$. The main idea is to plug in the pivot item $b$ in every group $\cG_g$ and  estimate the pivot-preference score $p_{ib} = Pr(i \succ b)$ of every item $i \notin \cG_g\setminus\{b\}$, i.e.  with respect to the pivot item $b$. We finally output the ranking simply sorting the items w.r.t. $p_{ib}$ -- the intuition is if item $i$ beats $j$ in terms of their actual BTL scores (i.e. $\theta_i > \theta_j$), then $i$ beats $j$ in terms of their pivot-preference scores as well (i.e. $p_{ib} > p_{jb}$).  

More formally, as $\bsigma$ denotes the ranking returned by \algant, the algorithm fails if $\bsigma$ is not \ebr. We denote by $Pr_b(\cdot) = Pr(\cdot \big | b \text{ is } \epsilon_b \text{-\emph{Best-Item}})$ the probability of an event conditioned on the event that $b$ is indeed an $\epsilon_b$-\emph{Best-Item} (Recall we have set $\epsilon_b = \min(\frac{\epsilon}{2},\frac{1}{2})$). Formally, we have:

\begin{align}
\label{eq:prf_ant1}
Pr_b(\text{ \algant\, fails }) & = Pr_b(\exists i,j \in [n] \mid \theta_i > \theta_j + \epsilon \text{ but } \sigma(i) > \sigma(j) )\\
\nonumber & = Pr_b(\exists i,j \in [n] \mid \theta_i > \theta_j + \epsilon \text{ but } \hp_{ib} < \hp_{jb} )
\end{align}
 
Now, assuming $b$ to be indeed an $\epsilon_b$-\emph{Best-Item}, since $\theta_i > \theta_j \implies \Delta_{ij}^b \ge  \frac{\epsilon}{8}$ (from Lemma \ref{lem:p_1i}), from Eqn. \ref{eq:prf_ant1}, we further get:

\begin{align}
\label{eq:prf_ant2}
\nonumber Pr_b(\text{ \algant\, fails }) & = Pr_b(\exists i,j \in [n] \mid \theta_i > \theta_j + \epsilon \text{ but } \sigma(i) > \sigma(j) )\\
\nonumber & \le Pr_b\Big(\exists i,j \in [n]\setminus\{b\} \mid p_{ib} > p_{jb} + \frac{\epsilon}{8} \text{ but } \sigma(i) > \sigma(j) \Big)\\
& = Pr_b(\exists i,j \in [n]\setminus \{b\} \mid \Delta_{ij}^b >  \frac{\epsilon}{8} \text{ but } \hp_{ib} < \hp_{jb} ),
\end{align}
where the inequality follows due to Lemma \ref{lem:p_1i}. In the inequality of the above analysis, it is also crucial to note that under the assumption of $b$ to be indeed an $\epsilon_b$-\emph{Best-Item} setting $\sigma(1) = b$ does not incur an error since $\theta_b > \theta_1 - \frac{\epsilon}{2}$.
So if we can estimate each $p_{ib}$ within a confidence interval of $\frac{\epsilon}{16}$, that should be enough to ensure correctness of the algorithm.
Thus the only thing remaining to show is \algant\, indeed estimates $p_{ib}$ tightly enough with high confidence -- formally, it is enough to show that for any group $g \in [G]$ and any item $i \in \cG_g \setminus \{b\}$, $Pr_b\Big(|p_{ib} - \hp_{ib}| > \frac{\epsilon}{16} \Big) \le \frac{\delta}{4n}$.

We prove this using the following two lemmas.
We first show that in any set $\cG_g$, if it is played for $m$ times, then with high probability of at least $(1-\delta)$, the pivot item would gets selected at least for $\frac{m}{4k}$ times. Formally:

\begin{lem}
\label{lem:ant_pivpickup} 
Conditioned on the event that $b$ is indeed an $\epsilon_b$-\emph{Best-Item}, for any group $g \in [G]$ with probability at least $\Big(1-\dfrac{\delta}{8n}\Big)$, the empirical win count $w_{b} > (1-\eta)\frac{t}{2k}$, for any $\eta \in \big(\frac{1}{8\sqrt 2},1 \big]$.
\end{lem} 

\begin{proof}
We will assume the event that $b$ to be indeed an $\epsilon_b$-\emph{Best-Item} throughout the proof and use the shorthand notation $Pr_b(\cdot)$ as defined earlier.
The proof now follows from an straightforward application of Chernoff-Hoeffding's inequality \cite{CI_book}. 
Recall that the algorithm plays each set $\cG_g$ for $t = \frac{2k}{\epsilon'^2}\ln \frac{1}{\delta'}$ number of times. Now consider a fixed group $g \in [G]$ and let $i_\tau$ denotes the winner of the $\tau$-{th} play of $\cG_g$, $\tau \in [t]$. Clearly, for any item $i \in \cG_g$, $w_i = \sum_{\tau = 1}^{t}\1(i_\tau == i)$, where $\1(i_\tau == i)$ is a Bernoulli random variable with parameter $\frac{\theta_i}{\sum_{j \in \cG_g}\theta_j}$, $\forall \tau \in [t]$, just by the definition of the PL query model with winner information (WI)  (Sec. \ref{sec:feed_mod}). Thus the random variable $w_i \sim Bin\Big(t,\frac{\theta_i}{\sum_{j \in \cG_g}\theta_j}\Big)$.
In particular, for the pivot item, $i = b$, we have $Pr_b(\{i_\tau = b\}) = \frac{\theta_b}{\sum_{j \in \cG_g}\theta_j} \ge \dfrac{ \frac{1}{2}}{k}$. Hence $\E[w_{b}] = \sum_{\tau = 1}^{t}\E[\1(i_\tau == b)] \ge \frac{t}{2k}$. 
Now applying multiplicative Chernoff-Hoeffdings bound for $w_b$, we get that for any $\eta \in \Big(\frac{1}{8},1\Big ]$, 

\begin{align*}
Pr_b\Big( w_b \le (1-\eta)\E[w_b] \Big) & \le \exp\Big(- \frac{\E[w_b]\eta^2}{2}\Big) \le \exp\Big(- \frac{t\eta^2}{4k}\Big), ~\Big(\text{since } \E[w_{b}] \ge \frac{t}{2k}\Big) 
 \\
& \le \exp\bigg(- \frac{\eta^2}{2\epsilon'^2} \ln \bigg( \frac{1}{\delta'} \bigg) \bigg) \le \exp\bigg(- \ln \bigg( \frac{1}{\delta'} \bigg) \bigg) \le \frac{\delta}{8n},
\end{align*}
where the second last inequality holds for any $\eta > \frac{1}{8\sqrt 2}$ as it has to be the case that $\epsilon' < \frac{1}{16}$ since $\epsilon \in (0,1)$.  Thus for any $\eta > \frac{1}{8 \sqrt 2}$, $\eta^2 \ge 4\epsilon'^2$.
\end{proof}

In particular, choosing $\eta = \frac{1}{2}$ in Lemma \ref{lem:ant_pivpickup}, we have with probability at least $\Big(1-\dfrac{\delta}{8n}\Big)$, the empirical win count of the pivot element $b$ is at least $w_{b} > \frac{t}{4k}$. We next proof under $w_b > \frac{t}{4k}$,  the estimate of pivot-preference scores $p_{ib}$ can not be too bad for any item $i \in \cG_g$ at any group $g \in [G]$. The formal statement is given in Lemma \ref{lem:piv_pib_conf}. For the ease of notation we define the event $\cE_g : = \{\exists i \in \cG_g \setminus\{b\} \text{ s.t. } |p_{ib} - \hp_{ib}| > \frac{\epsilon}{16} \}$.


\begin{lem}
\label{lem:piv_pib_conf} 
Conditioned on the event that $b$ is indeed an $\epsilon_b$-\emph{Best-Item}, for any group $g \in [G]$, $Pr_b\Big(\cE_g \Big) \le \frac{k\delta}{4n}$.
\end{lem}

\begin{proof}
We will again assume the event that $b$ to be indeed an $\epsilon_b$-\emph{Best-Item} throughout the proof and use the shorthand notation $Pr_b(\cdot)$ as defined previously.
Let us first fix a group $g \in [G]$.
We find convenient to define the event $\cF_g = \{ w_b \ge \frac{t}{4k} \text{ for group } \cG_g \}$ and denote by $n^g_{ib} = w_{i} + w_{b}$ the total number of times item $i$ and $b$ has won in group $\cG_g$. Clearly, $n^g_{ib} \le t$, moreover under $\cF_g$, $n^g_{ib} \ge \frac{t}{4k}, \, \forall i \in \cG_g$. Then for any item $i \in \cG_g\setminus\{b\}$,

\begin{align}
\label{eq:piv_long}
\nonumber Pr_b\Bigg( \bigg\{ |p_{ib}  - \hp_{ib}| > \frac{\epsilon}{16}  \bigg\} \cap \cF_g \Bigg) & \le  Pr_b\Bigg( \bigg\{ |p_{ib} - \hp_{ib}| > \frac{\epsilon}{16}  \bigg\} \cap \bigg\{ n^g_{ib} \ge \frac{t}{4k} \bigg\} \Bigg)\\
& \le 2\exp\Big( -2\frac{t}{4k}\big(\frac{\epsilon}{16}\big)^2 \Big)  = \frac{\delta}{4n},
\end{align}
where the first inequality follows since $\cF_g \implies n^g_{ib} \ge \frac{t}{4k}$, the second inequality holds due to Lemma \ref{lem:pl_simulator} with $\eta = \frac{\epsilon}{16}$, and $v = \frac{t}{4k}$. Its crucial to note that while applying \ref{lem:pl_simulator}, we can so we can drop the notation $Pr_b(\cdot)$ as
the event $\{|p_{ib} - \hp_{ib}| > \frac{\epsilon}{16}\} \cap \bigg\{ n^g_{ib} \ge \frac{t}{4k} \bigg\} $ is independent of $b$ to be $\epsilon_b$-\emph{Best-Item} or not.

Then probability that \algant\, fails to estimate the pivot-preference scores $p_{ib}$ for group $\cG_g$,

\begin{align*}
Pr_b(\cE_g) & = Pr_b(\cE_g \cap \cF_g) + Pr_b(\cE_g \cap \cF_g^c)\\
& \le Pr_b(\cE_g \cap \cF_g) +  Pr_b(\cF_g^c) \\
& \le Pr_b(\cE_g \cap \cF_g) +  \frac{\delta}{8n}  ~~\Big(\text{From Lemma \ref{lem:ant_pivpickup}}\Big)\\
& \le \sum_{i \in \cG_g\setminus\{b\}}Pr_b\Bigg( \bigg\{ |p_{ib} - \hp_{ib}| > \frac{\epsilon}{16}  \bigg\} \cap \cF_g \Bigg) +  \frac{\delta}{8n}  \\
& = (k-1)\frac{\delta}{4n} + \frac{\delta}{4n} \le \frac{k\delta}{4n},
\end{align*}
where the last inequality follows by taking union bound. The last equality follows from \eqref{eq:piv_long} and hence the proof follows.
\end{proof}

Thus using Lemma \ref{lem:piv_pib_conf} and from \eqref{eq:prf_ant2} we get,
\begin{align}
\label{eq:piv_sigerr}
\nonumber Pr_b( & \text{\algant\, fails })  \le Pr_b(\exists i,j \in [n] \setminus\{b\}\mid \Delta_{ij}^b >  \frac{\epsilon}{8} \text{ but } \hp_{ib} < \hp_{jb} )\\
 & \le Pr_b(\exists g \in [G] \text{ s.t. } \cE_g ) = \Big(\bigg \lceil\frac{n-1}{k-1} \bigg \rceil \Big)\frac{k\delta}{4n} \le 2\Big(\frac{n-1}{k-1}\Big)\frac{k\delta}{4n} \le  \frac{\delta}{2},
\end{align}
where the last inequality follows taking union bound over all groups $g \in [G]$.
Finally analysing all the previous claims together:

\begin{align*}
& Pr( \text{\algant\, fails })\\
& \le Pr(\text{\algant\, fails } \big | b \text{ is an } \epsilon_b\text{-\emph{Best-Item}})Pr(b \text{ is an } \epsilon_b\text{-\emph{Best-Item}}) + Pr(b \text{ is not an } \epsilon_b\text{-\emph{Best-Item}})\\
& \le Pr_b(\text{\algant\, fails })\Big(1-\frac{\delta}{2}\Big) + \frac{\delta}{2} \le \frac{\delta}{2} + \frac{\delta}{2} = \delta,
\end{align*}
where the last inequality follows from \eqref{eq:piv_sigerr}, which concludes the proof.
\end{proof}

\subsection{Proof of Theorem \ref{thm:est_thet}}

\ubthet*

\begin{proof}
Before proving the sample complexity, we first show the correctness of the algorithm, i.e. \algant\, is indeed \pac.
The following lemma would be crucially used throughout the proof analysis. Let us first denote $\theta_i^b = \frac{\theta_i}{\theta_b}$ the score of item $i \in [n]$ with respect to that of item $b$, we will term it as \emph{pivotal-score} of item $i$. Also let $\theta^b_{ij} = \theta_i^b - \theta_j^b, \,$ for any $i$ and $j \in [n]$. It is easy to note that since with high probability $(1-\frac{\delta}{4})$, $\frac{1}{2} \le \theta_b \le 1$ (Lemma \ref{lem:trace_best}), and hence $\theta_i \le \theta_i^b \le 2\theta_i$. This further leads to the following claim:

\begin{lem}
\label{lem:theta_1i}
If $b$ is the pivot-item returned by Algorithm \ref{alg_bi} (Line $5$ of \algestscr), then
for any two items $i,j \in [n]$, such that $\theta_i \ge \theta_j$, $(\theta_i - \theta_j) \le \theta^b_{ij} \le 2(\theta_i - \theta_j)$, with probability at least $\big( 1 - \frac{\delta}{4}\big)$.
\end{lem}

\begin{proof}
First let us assume if $b = 1$, which implies $\theta_i^b = \frac{\theta_i}{\theta_b} = \theta_i$ and the claims holds trivially.

Now let us assume that $b \neq 1$, but by Lemma \ref{lem:trace_best}, with high probability, $\theta_b \ge \theta_1 - \frac{1}{2} = \frac{1}{2}$ as \algwin\, returns an $\epsilon_b$-\emph{Best-Item} with probability at least $(1-\frac{\delta}{4})$. Also $\theta_b \le 1$ for any $b \neq 1$. The above bounds on $\theta_b$ clearly implies $\theta^b_{ij} = \frac{\theta_i-\theta_j}{\theta_b} \in \Big[ (\theta_i-\theta_j), 2(\theta_i-\theta_j) \Big]$.
\end{proof}

Now to ensure the correctness of \algant, recall that all we need to show it returns an \ebr\ $\bsigma \in \Sigma_{[n]}$. Same as Algorithm \ref{alg:alg_ant}, here also we plug in the pivot item $b$ in every group $\cG_g$. But now it estimates the \emph{pivotal score} $\theta_{i}^b$ of every item $i \notin \cG_g\setminus\{b\}$ instead of \emph{pivotal preference score} $p_{ib}$ where lies the uniqeness of \algestscr. We finally output the ranking simply sorting the items w.r.t. $\theta_{i}^b$ -- the intuition is if item $i$ beats $j$ in terms of their actual BTL scores (i.e. $\theta_i > \theta_j$), then $i$ beats $j$ in terms of their \emph{pivotal scores} as well (i.e. $\theta_{i}^b > \theta_{j}^b$).  

More formally, as $\bsigma$ denotes the ranking returned by \algant, the correctness of algorithm \algestscr\, fails if $\bsigma$ is not an \ebr. 
 Formally, we have:

\begin{align}
\label{eq:prf_estscr1}
Pr(\text{ Correctness of \algant\, fails }) &= Pr(\exists i,j \in [n] \mid \theta_i > \theta_j + \epsilon \text{ but } \sigma(i) > \sigma(j) )\\
\nonumber & = Pr(\exists i,j \in [n] \mid \theta_i > \theta_j + \epsilon \text{ but } \htheta_i^b < \htheta_j^b )
\end{align}
 
Now, assuming $b$ to be indeed an $\epsilon_b$-\emph{Best-Item}, since $\theta_i > \theta_j \implies \theta_{ij}^b \ge  \epsilon$ (from Lemma \ref{lem:theta_1i}), from Eqn. \ref{eq:prf_estscr1}, we further get:

\begin{align}
\label{eq:prf_estscr2}
\nonumber Pr(\text{ Correctness of \algant\, fails }) & = Pr(\exists i,j \in [n] \mid \theta_i > \theta_j + \epsilon \text{ but } \sigma(i) > \sigma(j) )\\
\nonumber & \le Pr\Big(\exists i,j \in [n]\setminus\{b\} \mid \theta_{i}^b > \theta_{j}^b + \epsilon \text{ but } \sigma(i) > \sigma(j) \Big)\\
& = Pr(\exists i,j \in [n]\setminus \{b\} \mid \theta_{i}^b > \theta_{j}^b + \epsilon \text{ but } \htheta_i^b < \htheta_j^b  ),
\end{align}
where the inequality follows due to Lemma \ref{lem:theta_1i}. In the inequality of the above analysis, it is also crucial to note that under the assumption of $b$ to be indeed an $\epsilon_b$-\emph{Best-Item} setting $\sigma(1) = b$ does not incur an error since $\theta_b > \theta_1 - \frac{\epsilon}{2}$.
So if we can estimate each $\htheta_i^b$ within a confidence interval of $\frac{\epsilon}{2}$, that should be enough to ensure correctness of the algorithm.
Thus the only thing remaining to show is \algestscr\, indeed estimates $\theta_i^b$ tightly enough with high confidence -- formally, it is enough to show that for any group $g \in [G]$ and any item $i \in \cG_g \setminus \{b\}$, $Pr\Big(|\theta_i^b - \htheta_i^b| > \frac{\epsilon}{2} \Big) \le \frac{\delta}{4n}$. 

For this we will be crucially using the result of Lemma \ref{lem:geo2mgf}, which shows that $\htheta_i^b \sim Geo\Big(\dfrac{\theta_b}{\theta_i + \theta_b}\Big)$ for any item $i \in \cG_g$ and at any group $g \in [G]$. 

Using the above insight, we will show that the estimate of \emph{pivotal scores} $\htheta_i^b$ can not be too bad for any item $i \in \cG_g$ at any group $g \in [G]$. The formal statement is given in Lemma \ref{lem:piv_pib_conf}. For the ease of notation let us define the event $\cE_g : = \{\exists i \in \cG_g \setminus\{b\} \text{ s.t. } |\theta_{i}^b - \htheta_{i}^b| > \frac{\epsilon}{2} \}$.

\begin{lem}
\label{lem:estscr_theta_conf} 
For any group $g \in [G]$, $Pr\Big(\cE_g \Big) \le \frac{(k-1)\delta}{4n}$.
\end{lem}

\begin{proof}
We will again assume the event that $b$ to be indeed an $\epsilon_b$-\emph{Best-Item} throughout the proof and use the shorthand notation $Pr_b(\cdot)$ as defined previously. 
Let us first fix a group $g \in [G]$. Then for any item $i \in \cG_g\setminus\{b\}$,

\begin{align}
\label{eq:estscr_long}
Pr\Big( |\htheta_{i}^b - \theta_{i}^b| > \frac{\epsilon}{2}\Big) \le  2\exp\Bigg( - \frac{2t(\epsilon/2)^2}{\Big( 1+\frac{\theta_i}{\theta_b} \Big)^2\Big((\epsilon/2) + 1+\frac{\theta_i}{\theta_b}  \Big)}\Bigg) = \frac{\delta}{4n},
\end{align}
where the inequality follows from Lemma \ref{lem:geoconc}.
Then summing over, all the items $i \in \cG_g\sm\{b\}$ in group $g \in [G]$,

\begin{align*}
Pr_b(\cE_g) & \le \sum_{i \in \cG_g\setminus\{b\}}Pr\Bigg( |\htheta_{i}^b - \theta_{i}^b| > \frac{\epsilon}{2} \Bigg) = (k-1)\frac{\delta}{4n},
\end{align*}
where the inequality follows from \eqref{eq:estscr_long}.
\end{proof}

Now applying Lemma \ref{lem:estscr_theta_conf} over all groups $g \in [G]$, and using \eqref{eq:prf_estscr2}, the probability that correctness of \algestscr \, fails:

\begin{align}
\label{eq:estscr_crr}
\nonumber Pr( & \text{ Correctness of \algant\, fails })  \le Pr(\exists i,j \in [n]\setminus \{b\} \mid \theta_{i}^b > \theta_{j}^b + \epsilon \text{ but } \htheta_i^b < \htheta_j^b  )\\
 & \le Pr_b(\exists g \in [G] \text{ s.t. } \cE_g ) = \Big(\bigg \lceil\frac{n-1}{k-1} \bigg \rceil \Big)\frac{(k-1)\delta}{4n} \le 2\Big(\frac{n-1}{k-1}\Big)\frac{(k-1)\delta}{4n} \le  \frac{\delta}{2},
\end{align}
where the last inequality follows taking union bound over all groups $g \in [G]$.

Thus we are now only left to prove that the correctness of \algestscr\, indeed holds within the desired sample complexity of $O\big(\frac{n}{\epsilon^2} \log \frac{n}{\delta}\big)$. Towards this, let us first define $t' = \frac{5}{2}tk = \frac{5*567k}{2\epsilon^2}\ln\big( \frac{8n}{\delta} \big)$.
Also let $Pr_b(\cdot) = Pr(\cdot \big | b \text{ is } \epsilon_b \text{-\emph{Best-Item}})$ denotes the probability of an event conditioned on the event that $b$ is indeed an $\epsilon_b$-\emph{Best-Item} (Recall we have set $\epsilon_b = \min(\frac{\epsilon}{2},\frac{1}{2})$).
For any group $g \in [G]$, we denote by $\cT_g$ the total number of times $\cG_g$ was played until $t$ wins of item $b$ were observed. Also recall the probability of item $b$ being the winner at any subsetwise play of $\cC_g$ is given by $p := \frac{\theta_b}{\sum_{j \in \cG_g}\theta_j} > \frac{1}{2k}$, given $b$ is indeed an $\epsilon_b$-\emph{Best-Item} (from Lemma \ref{lem:trace_best}). So we have the $\cT_g \sim $ \emph{NB}$(t,\frac{\theta_b}{\sum_{j \in \cG_g}\theta_j})$.
Then for any fixed group $g \in [G]$, the probability that $\cG_g$ needs to be played (queried) for more than $t'$ times to get at least $t$ wins of item $b$:

\begin{align*}
Pr_b(\cT_g > t') & = Pr\Big( Bin(t',p) < t \Big) \\
& \le Pr\Big( Bin(t',p) < \frac{4}{5}pt' \Big) ~~\Bigg[\text{ Since } pt' > \frac{5tk}{4k}\Bigg] \\
& = Pr\Big( Bin(t',p) - pt' < -\frac{1}{5}pt' \Big)\\
& \le Pr\Big( Bin(t',p) - pt' < (1-\frac{4}{5})pt' \Big)\\
& \le \exp\Big(- \frac{pt'(4/5)^2}{2}\Big) ~~\Bigg[\text{ By multiplicative Chernoff bound } \Bigg] \\
& \le \exp\Big( -\frac{2}{5}t \Big) ~~\Bigg[\text{ Using } pt' > \frac{5t}{4}\Bigg] \\
& \le \frac{\delta}{8n} ~~\Bigg[ \text{ Recall we have } t = \frac{567}{4\big(\epsilon/2\big)^2}\ln\big( \frac{8n}{\delta}\big) \Bigg]
\end{align*}

Then taking union bound over all the groups

\begin{align}
\label{eq:estscr_sc}
\nonumber Pr(\exists g \in [G] \mid \cT_g > t')
 & \le Pr_b(\exists g \in [G] \mid \cT_g > t')Pr(\theta_b > \theta_1 - \epsilon_b) + Pr(\theta_b < \theta_1 - \epsilon_b)\\
\nonumber & \le Pr_b(\exists g \in [G] \mid \cT_g > t') + \frac{\delta}{4} \le \sum_{g \in [G]}Pr_b(\cT_g > t') + \frac{\delta}{4}\\
& \le \Big( \lceil \frac{n-1}{k-1} \rceil \Big)\frac{\delta}{8n} + \frac{\delta}{4} \le 2\Big( \frac{n-1}{k-1} \Big)\frac{\delta}{8n} + \frac{\delta}{4} \le \frac{\delta}{2}.
\end{align}

Then with high probability $(1-\frac{\delta}{2})$, $\forall g \in [G]$ we have $\cT_g < t' = \frac{5}{2}tk$, which makes the total sample complexity of \algestscr\, to be at most $\Big\lceil \frac{n-1}{k-1}\Big\rceil t' \le \frac{2n}{k}\frac{5tk}{2} = \frac{2835n}{\epsilon^2}\ln  \frac{8n}{\delta}  = O\bigg( \frac{n}{\epsilon^2} \ln \frac{n}{\delta} \bigg) $,
as total number of groups are $G = \lceil \frac{n-1}{k-1} \rceil$.

Moreover the sample complexity of \algwin\, is also $O\Big( \frac{n}{\epsilon^2}\log \frac{n}{\delta}\Big)$ as proved in Lemma \ref{lem:trace_best}. Combining this with above the total sample complexity of \algant \, remains $O\Big( \frac{n}{\epsilon^2}\log \frac{n}{\delta}\Big)$.
Finally analysing all the previous claims together:
\begin{align*}
& Pr( \text{\algant\, fails })\\
& \le Pr(\text{ Correctness of \algant\, fails }) + Pr(\text{ Sample complexity of \algant\, fails })\\
& \le \frac{\delta}{2} + \frac{\delta}{2} = \delta,
\end{align*}
where the last inequality follows from \eqref{eq:estscr_crr} and \eqref{eq:estscr_sc}, which concludes the proof.
\end{proof}

\section{Appendix for Section \ref{sec:lb_wi}}
\label{app:wilb}

\subsection{Restating Lemma 1 of  \cite{Kaufmann+16_OnComplexity}}
\label{app:gar}

Consider a multi-armed bandit (MAB) problem with $n$ arms. At round $t$, let $A_t$ and $Z_t$ denote the arm played and the observation (reward) received, respectively. Let $\cF_t = \sigma(A_1,Z_1,\ldots,A_t,Z_t)$ be the sigma algebra generated by the trajectory of a sequential bandit algorithm upto round $t$.
\begin{restatable}[Lemma $1$, \cite{Kaufmann+16_OnComplexity}]{lem}{gar16}
\label{lem:gar16}
Let $\nu$ and $\nu'$ be two bandit models (assignments of reward distributions to arms), such that $\nu_i ~(\text{resp.} \,\nu'_i)$ is the reward distribution of any arm $i \in \cA$ under bandit model $\nu ~(\text{resp.} \,\nu')$, and such that for all such arms $i$, $\nu_i$ and $\nu'_i$ are mutually absolutely continuous. Then for any almost-surely finite stopping time $\tau$ with respect to $(\cF_t)_t$,
\vspace*{0pt}
\begin{align*}
\sum_{i = 1}^{n}\E_{\nu}[N_i(\tau)]KL(\nu_i,\nu_i') \ge \sup_{\cE \in \cF_\tau} kl(Pr_{\nu}(\cE),Pr_{\nu'}(\cE)),
\end{align*}
where $kl(x, y) := x \log(\frac{x}{y}) + (1-x) \log(\frac{1-x}{1-y})$ is the binary relative entropy, $N_i(\tau)$ denotes the number of times arm $i$ is played in $\tau$ rounds, and $Pr_{\nu}(\cE)$ and $Pr_{\nu'}(\cE)$ denote the probability of any event $\cE \in \cF_{\tau}$ under bandit models $\nu$ and $\nu'$, respectively.
\end{restatable}

\subsection{Proof of Lemma \ref{lem:lb_sym}}
\label{app:lb_kl}

\symlem*

\begin{proof}
\textbf{Base Case.} The claim follows trivially for $q = 1$, just from the definition of \pac\ property of the algorithm. Suppose $S = \{i\}$ for some $i \in [n-1]$. Then clearly 
\[
Pr_{S}\Big( \bsigma_{\cA}(1)  = \{0\} \Big) < Pr_{S}\Big( \bsigma_{\cA}(1)  \neq \{i\} \Big) < \delta.
\]

So for the rest of the proof we focus only on the regime where $2 \le q \le n-1$. 
Let us first fix an $q' \in [n-2]$ and set $q = q'+1$. Clearly $2 \le q \le n-1$. Consider a problem instance $\nu_S \in \bnu_{[q]}$. Recall from Remark \ref{rem:inst_not} that we use the notation $S \in \bnu_{[q]}$ to denote a problem instance in $\bnu_{[q]}$. 
Then probability of doing an error over all possible choices of $S \in \bnu_{[q]}$:

\begin{align}
\label{eq:lb_identity}
\nonumber \sum_{S \in \bnu_{[q]} }Pr_{S}\Big( \bsigma_{\cA}(1:q) \neq S \Big) & \ge \sum_{S \in \bnu_{[q'+1]}}\sum_{i \in S}Pr_{S}\Big( \bsigma_{\cA}(1:q) = S\sm\{i\}\cup\{0\} \Big)\\
& = \sum_{S' \in \bnu_{[q']}}\sum_{i \in [n-1]\setminus S'}Pr_{S' \cup \{i\}}\Big( \bsigma_{\cA}(1:q) = S' \cup\{0\} \Big),
\end{align}
where the above analysis follows from a similar result proved by \cite{Kalyanakrishnan+12} to derive sample complexity lower bound for classical multi-armed bandit setting towards recovering top-$q$ items (see Theorem $8$, \cite{Kalyanakrishnan+12}).

Clearly the possible number of instances in $\bnu_{[q']}$, i.e. $|\bnu_{[q']}| = {{n-1}\choose{q'}}$, as any set $S \subset [n-1]$ of size $q'$ can be chosen from $[n-1]$ in ${{n-1}\choose{q'}}$ ways. Similarly, $|\bnu_{[q]}| = {{n-1}\choose{q}} = {{n-1}\choose{q'+1}}$.

Now from \emph{symmetry} of algorithm $\cA$ and by construction of the class of our problem instances $\bnu_{[q']}$, for any two instances $S'_1$ and $S'_2$ in $\bnu_{[q']}$, and for any choices of $i \in [n-1]\setminus S'_1$ and $j \in [n-1]\setminus S'_2$ we have that:

\[
Pr_{S'_1 \cup \{i\}}\Big( \bsigma_{\cA}(1:q) = S'_1 \cup \{0\} \Big) = Pr_{S'_2 \cup \{j\}}\Big( \bsigma_{\cA}(1:q) = S'_2 \cup \{0\} \Big).
\]

Let us denote $Pr_{S'_1 \cup \{i\}}\Big( \bsigma_{\cA} = \bsigma_{S'_1} \Big) = p' \in (0,1)$. Then 
above equivalently implies that for all $S \in \bnu_{[q]}$ and any $i \in [n-1]\setminus S$,
\[
Pr_{S}(\bsigma_{\cA}(1:q) = S \setminus \{i\} \cup\{0\}) = p'.
\]
Then using above in \eqref{eq:lb_identity} we can further derive,

\begin{align*}
\sum_{S \in \bnu_{[q]} }Pr_{S}\Big( \bsigma_{\cA}(1:q) \neq  {S} \Big) & \ge  \sum_{S' \in \bnu_{[q']}}\sum_{i \in [n-1]\setminus S'}Pr_{S' \cup \{i\}}\Big( \bsigma_{\cA}(1:q) = S' \cup\{0\} \Big)\\
& = \sum_{S' \in \bnu_{[q']}}(n-1-q')p'\\
& = {{n-1}\choose{q'}}(n-1-q')p'\\
& = {{n-1}\choose{q'}}\frac{n-1-q'}{q'+1}(q+1)p'\\
& = {{n-1}\choose{q'+1}}(q'+1)p' = {{n-1}\choose{q}}qp' .
\end{align*}
But now observe that $|\bnu_{[q]}| = {{n-1}\choose{q}}$. Thus if $p' \ge \frac{\delta}{q}$, this implies 

\begin{align*}
\sum_{S \in \bnu_{[q]} }Pr_{S}\Big( \bsigma_{\cA}(1:q) \neq  {S} \Big) \ge {{n-1}\choose{q}}{\delta},
\end{align*}
which in turn implies that there exist at least one instance $\nu_S \in \bnu_{[q+1]}$ such that $Pr_{S}\Big( \bsigma_{\cA}(1:q) \neq  S \Big) \ge \delta$, which violates the \pac\ property of algorithm $\cA$. Thus it has to be the case that $p' < \frac{\delta}{q}$.
Recall that we have chosen $2 \le q \le n-1$, and for any $S \in \bnu_{[q]}$, and any $i \in [n-1]\setminus S$ we proved that 
\[
Pr_{S}(\bsigma_{\cA}(1:q) = S \setminus \{i\} \cup\{0\}) = p' < \frac{\delta}{q},
\]
which concludes the proof.
\end{proof}

\subsection{Proof of Lemma \ref{lem:kl_del}}

\begin{restatable}[]{lem}{kldel}
\label{lem:kl_del}
For any $\delta \in (0,1)$, and $q \in \R_+$,
$
kl\bigg(1-\delta,\dfrac{\delta}{q}\bigg) > \ln \dfrac{q}{4\delta}.
$ 
\end{restatable}

\begin{proof}
The proof simply follows from the definition of KL divergence. Recall that for any $p_1, p_2 \in (0,1)$,
\[
kl(p_1,p_2) = p_1 \ln \frac{p_1}{p_2} + (1-p_1) \ln \frac{1-p_1}{1-p_2}.
\]

Applying above in our case we get,
\begin{align*}
kl\bigg(1-\delta,\frac{\delta}{q}\bigg) & = (1-\delta)\ln \frac{q(1-\delta)}{\delta} + \delta \ln \frac{q(\delta)}{(q-\delta)}\\
& = (1-\delta)\ln \frac{q(1-\delta)}{\delta} + \delta \ln \frac{q(\delta)}{(q-\delta)}\\
& = \ln q + (1-\delta)\ln \frac{(1-\delta)}{\delta} + \delta \ln \frac{(\delta)}{(q-\delta)}\\
& \ge \ln q + (1-\delta)\ln \frac{(1-\delta)}{\delta} + \delta \ln \frac{(\delta)}{(1-\delta)} ~~[\text{ since } q \ge 1]\\
& = \ln q + (1-2\delta)\ln \frac{(1-\delta)}{\delta}\\
& \ge \ln q + (1-2\delta)\ln \frac{1}{2\delta} ~~\Big[\text{ since the second term is negative for } \delta \ge \frac{1}{2}\Big]\\
& = \ln q + \ln \frac{1}{2\delta} + 2\delta\ln 2\delta\\
& \ge \ln q + \ln \frac{1}{2\delta} + 2\delta\Big(1-\frac{1}{2\delta}\Big) ~~ \Big[ \text{ since } x\ln x \ge (x-1), \, \forall x > 0 \Big]\\
& =  \ln q + \ln \frac{1}{2\delta} -(1 - 2\delta)\\
& \ge \ln q + \ln \frac{1}{2\delta} + \ln \frac{1}{2} = \ln \frac{q}{4\delta}.
\end{align*}
\end{proof}

\subsection{Proof of Theorem \ref{thm:lb_plpac_win}}
\lbwin*

\begin{proof}
The main idea lies in constructing `hard enough' problem instances on which no algorithm can be \pac\ without observing $\Omega\bigg( \frac{n}{\epsilon^2} \ln \frac{1}{4\delta}\bigg)$ number of samples. We crucially use the results of \cite{Kaufmann+16_OnComplexity} (Lemma \ref{lem:gar16}) for the purpose.

Towards this we first fix our \emph{true problem instance} ($\nu$ in Lemma \ref{lem:gar16}) to be $\nu_{S^*} \in \bnu_{[q]}$, for some $q \in [n-2]$ (the actual value of $q$ to be decided later).
Note that the arm set $\cB$ (of Lemma \ref{lem:gar16}) for our current problem setup is set of all $k$-sized subsets of $[n-1]\cup\{0\}$, i.e. $\cB = \{ S \subseteq [n-1]\cup\{0\} \mid |S| = k \}$.

We now fix the altered problem instance ($\nu'$ in Lemma \ref{lem:gar16}) to be $\nu_{\tS^*} \in \bnu_{[q+1]}$ such that $\tS^* = S^* \cup \{a\}$, of some $a \in [n-1]\setminus S^*$. 
Now if $\bsigma_{\cA} \in \Sigma_{[n]}$ is the ranking returned by Algorithm $\cA$, then clearly owing to the \pac\ property of $\cA$,
\begin{align}
\label{eq:lbthm_1}
Pr_{\sS}\Big( \bsigma_{\cA}(1:q+1) = S^* \cup \{0\} \Big) > Pr_{\sS}\big(\{\bsigma)_{\cA} \text{ is an \ebr}\}\big)  > (1-\delta),
\end{align} 

Moreover the \pac\ property of $\cA$ also implies that:  

\[
Pr_{\tS^*}\Big( \bsigma_{\cA}(1:q+1) = S^* \cup \{0\} \Big) \le Pr_{\tS^*}\Big( \bsigma_{\cA}  \neq \bsigma_{\tS^*} \Big) < \delta.
\]

But owing to Lemma \ref{lem:lb_sym}, we are able to claim a stronger bound (using $S = \tS^*$):

\begin{align}
\label{eq:lbthm_2}
Pr_{\tS^*}\Big( \bsigma_{\cA}(1:q+1) = S^* \cup \{0\} \Big) < \frac{\delta}{q}.
\end{align}

We will crucially use \eqref{eq:lbthm_1} and \eqref{eq:lbthm_2} in the following analysis. But before that
note that for problem instance $\nu_{S^*} \in \bnu_{[q]}$, the probability distribution associated with a particular arm  (set of size $k$ in our case) $B \in \cB$ is given by:
\[
\nu^B_\sS \sim Categorical(p_1, p_2, \ldots, p_k), \text{ where } p_i = Pr(i|B), ~~\forall i \in [k], \, \forall B \in \cB,
\]
where $Pr(i|S)$ is as defined in Sec. \eqref{eq:prob_win}. 
Now applying Lemma \ref{lem:gar16}, for some event $\cE \in \cF_\tau$ we get,

\begin{align}
\label{eq:FI_a}
\sum_{\{B \in \cB : a \in B\}}\E_{\nu^B_\sS}[N_B(\tau_{\cA})]KL(\nu^B_\sS, \nu^B_{\tS^*}) \ge {kl(Pr_{\nu^B_\sS}(\cE),Pr_{\nu^B_{\tS^*}}(\cE))},
\end{align}
where $N_B(\tau_{\cA})$ denotes the number of times arm (subset of size $k$) $B$ is played by $\cA$ in $\tau$ rounds.
Above clearly follows due to the fact that for any arm $B \in \cB$ such that $a \notin B$, $\nu^B_{S}$ is same as $\nu^B_{\tS^*}$, and hence $KL(\nu^B_\sS, \nu^B_{\tS^*})=0$, $\forall S \in \cA, \,a \notin S$. 
For the notational convenience we will henceforth denote $\cB^a = \{B \in \cB : a \in S\}$. 

Now let us first analyse the right hand side of \eqref{eq:FI_a}, for any set $B \in \cB^a$. 

\textbf{Case 1.} Assume $0 \notin B$, and denote by $r = |B \cap \sS|$ the number of ``good" arms with PL parameter $\theta\Big(\frac{1}{2} + \epsilon\Big)^2$.
Note that for problem instance $\nu_{\sS}^B$,
\begin{align*}
\nu^B_{\sS}(i) = 
\begin{cases} 
\frac{\theta(\frac{1}{2}+\epsilon)^2}{r\theta(\frac{1}{2}+\epsilon)^2 + (k-r)\theta(\frac{1}{2}-\epsilon)^2} = \frac{R^2}{rR^2 + (k-r)}, \forall i \in [k], \text{ such that } B(i) \in S^*,\\
\frac{\theta(\frac{1}{2}-\epsilon)^2}{r\theta(\frac{1}{2}+\epsilon)^2 + (k-r)\theta(\frac{1}{2}-\epsilon)^2} = \frac{1}{rR^2 + (k-r)}, \text{ otherwise. }
\end{cases}
\end{align*}

Similarly, for problem instance $\nu_{\tS^*}^B$, we have: 

\begin{align*}
\nu^B_{\tS^*}(i) = 
\begin{cases} 
\frac{\theta(\frac{1}{2}+\epsilon)^2}{(r+1)\theta(\frac{1}{2}+\epsilon)^2 + (k-r-1)\theta(\frac{1}{2}-\epsilon)^2} = \frac{R^2}{(r+1)R^2 + (k-r-1)}, \forall i \in [k], \text{ such that } B(i) \in \tS^* = \sS \cup \{a\},\\
\frac{\theta(\frac{1}{2}-\epsilon)^2}{(r+1)\theta(\frac{1}{2}+\epsilon)^2 + (k-r-1)\theta(\frac{1}{2}-\epsilon)^2} = \frac{1}{(r+1)R^2 + (k-r-1)}, \text{ otherwise. }
\end{cases}
\end{align*}

Now using the following upper bound on $KL(\p_1,\p_2) \le \sum_{x \in \X}\frac{p_1^2(x)}{p_2(x)} -1$, $\p_1$ and $\p_2$ be two probability mass functions on the discrete random variable $\X$ \cite{klub16} we get:

\begin{align}
\label{eq:lb_kln0}
\nonumber KL(\nu^B_\sS, \nu^B_{\tS^*}) & \le \frac{(r+1)R^2 + (k-r-1)}{(rR^2 + k - r)^2}\Big [ rR^2 + \frac{1}{R^2} + (k - r - 1) \Big] -1\\
& = \Big( R - \frac{1}{R}\Big)^2\Big[ \frac{rR^2 + (k-r-1)}{(rR^2 + k - r)^2} \Big]
\end{align}

\textbf{Case 2.} Now assume $0 \in B$, and denote by $r = |B \cap \sS\cup \{0\}|$ the number of ``non-bad" arms with PL parameter greater than $\theta\Big(\frac{1}{2} - \epsilon\Big)^2$. Clearly $r \ge 1$ as $0 \in B$.
Similar to \textbf{Case 1}, for problem instance $\nu_{\sS}^B$,
\begin{align*}
\nu^B_{\sS}(i) = 
\begin{cases} 
\frac{\theta(\frac{1}{2}+\epsilon)^2}{(r-1)\theta(\frac{1}{2}+\epsilon)^2 + (k-r)\theta(\frac{1}{2}-\epsilon)^2 + \theta(\frac{1}{4}-\epsilon^2)} = \frac{R^2}{(r-1)R^2 + (k-r) + R}, \forall i \in [k], \text{ such that } B(i) \in S^*,\\
\frac{\theta(\frac{1}{4}-\epsilon^2)}{(r-1)\theta(\frac{1}{2}+\epsilon)^2 + (k-r)\theta(\frac{1}{2}-\epsilon)^2 + \theta(\frac{1}{4}-\epsilon^2)} = \frac{R}{(r-1)R^2 + (k-r) + R}, \forall i \in [k], \text{ such that } B(i) = 0,\\
\frac{\theta(\frac{1}{2}-\epsilon)^2}{(r-1)\theta(\frac{1}{2}+\epsilon)^2 + (k-r)\theta(\frac{1}{2}-\epsilon)^2 + \theta(\frac{1}{4}-\epsilon^2)} = \frac{1}{(r-1)R^2 + (k-r) + R}, \text{ otherwise. }
\end{cases}
\end{align*}

Similarly, for problem instance $\nu_{\tS^*}^B$, we have: 

\begin{align*}
\nu^B_{\tS^*}(i) = 
\begin{cases} 
\frac{\theta(\frac{1}{2}+\epsilon)^2}{r\theta(\frac{1}{2}+\epsilon)^2 + (k-r-1)\theta(\frac{1}{2}-\epsilon)^2 + \theta(\frac{1}{4}-\epsilon^2)} = \frac{R^2}{rR^2 + (k-r-1) + R}, \forall i \in [k], \text{ such that } B(i) \in \tS^* = \sS \cup \{a\},\\
\frac{\theta(\frac{1}{4}-\epsilon^2)}{r\theta(\frac{1}{2}+\epsilon)^2 + (k-r-1)\theta(\frac{1}{2}-\epsilon)^2 + \theta(\frac{1}{4}-\epsilon^2)} = \frac{R}{rR^2 + (k-r-1) + R}, \forall i \in [k], \text{ such that } B(i) = 0,\\
\frac{\theta(\frac{1}{2}-\epsilon)^2}{r\theta(\frac{1}{2}+\epsilon)^2 + (k-r-1)\theta(\frac{1}{2}-\epsilon)^2 + \theta(\frac{1}{4}-\epsilon^2)} = \frac{1}{rR^2 + (k-r-1) + R}, \text{ otherwise. }
\end{cases}
\end{align*}

Same as before, again using the following upper bound on $KL(\p_1,\p_2) \le \sum_{x \in \X}\frac{p_1^2(x)}{p_2(x)} -1$, $\p_1$ and $\p_2$ be two probability mass functions on the discrete random variable $\X$ \cite{klub16} we get:

\begin{align}
\label{eq:lb_kl0}
\nonumber KL(\nu^B_\sS, \nu^B_{\tS^*}) & \le \frac{rR^2 + R + (k-r-1)}{((r-1)R^2 + R + k - r)^2}\Big [ rR^2 + \frac{1}{R^2} + (k - r - 1) \Big] -1\\
& = \Big( R - \frac{1}{R}\Big)^2\Bigg[ \frac{ (r-1)R^2 + (k-r) + R  - 1}{\big( (r-1)R^2 + (k-r) + R \big)^2} \Bigg]
\end{align}

Now, consider $\cE_0 \in \cF_\tau$ be an event such that $\cE_0:= \{\bsigma_{\cA}(1:q+1) = S^*\cup\{0\}\}$. Note that since algorithm $\cA$ is \pac\,, clearly $Pr_{\nu^B_\sS}(\cE_0) > Pr_{\nu^B_\sS}\big(\{\bsigma)_{\cA} \text{ is an \ebr}\}\big) > (1-\delta)$. On the other hand, Lemma \ref{lem:lb_sym} implies that $Pr_{\nu^B_{\tS^*}}(\cE_0)) < \frac{\delta}{q}$.
Then analysing the left hand side of \eqref{eq:FI_a} for $\cE = \cE_0$ along with \eqref{eq:lbthm_1} and \eqref{eq:lbthm_2}, we get that

\vspace*{-10pt}
\begin{align}
\label{eq:win_lb2}
kl(Pr_{\nu^B_\sS}(\cE_0),Pr_{\nu^B_{\tS^*}}(\cE_0)) \ge kl\Big(1-\delta,\frac{\delta}{q}\Big) \ge \ln \frac{q}{4\delta}
\end{align}

where the last inequality follows from Lemma \ref{lem:kl_del}.

Now applying \eqref{eq:FI_a} for each altered problem instance $\nu^B_{\tS^*}$, each corresponding to any one of the $(n-1-q)$ different choices of $a \in [n-1]\setminus \sS$, and summing all the resulting inequalities of the form \eqref{eq:FI_a}:

\begin{align}
\label{eq:lb1}
\sum_{a \in [n-1]\setminus\sS}\sum_{B \in \cB^a}\E_{\nu^B_\sS}[N_B(\tau_{\cA})]KL(\nu^B_\sS,\nu^B_{\tS^*}) \ge (n-1-q)\ln \frac{q}{4\delta}.
\end{align}

A crucial observation here is that in the left hand side of \eqref{eq:lb1} above, any $B \in \cB^a$ shows up for exactly $k - r$ may times, where $r$ is as defined in \textbf{Case 1} and \textbf{Case 2} above. Thus, given a fixed set $B$, the coefficient of the term $\E_{\nu^B_\sS}$ becomes for:

\textbf{Case 1.} From \eqref{eq:lb_kln0}, $(k-r)\Big( R - \frac{1}{R}\Big)^2\Big[ \frac{rR^2 + (k-r-1)}{(rR^2 + k - r)^2} \Big] \le \Big( R - \frac{1}{R}\Big)^2$, as $r \ge 0$.

\textbf{Case 2.} From \eqref{eq:lb_kl0}, $(k-r)\Big( R - \frac{1}{R}\Big)^2\Bigg[ \frac{ (r-1)R^2 + (k-r) + R  - 1}{\big( (r-1)R^2 + (k-r) + R \big)^2} \Bigg] \le \Big( R - \frac{1}{R}\Big)^2$, as in this case $r \ge 1$, and note that $R = \frac{\frac{1}{2}+\epsilon}{\frac{1}{2}-\epsilon} > 1$ by definition.

Thus from \eqref{eq:lb1} we further get

\begin{align}
\label{eq:lb2}
\nonumber \sum_{a = 2}^{n}&\sum_{\{S \in \cA \mid a \in S\}} \E_{\nu^B_\sS}[N_B(\tau_A)]KL(\nu^B_\sS,\nu^B_{\tS^*}) \le \sum_{S \in \cA}\E_{\nu^B_\sS}[N_B(\tau_A)]\Bigg( R - \frac{1}{R} \Bigg)^2\\
& \le 256\epsilon^2\sum_{S \in \cA}\E_{\nu^B_\sS}[N_B(\tau_A)] ~~\Bigg[\text{since, } \Bigg( R - \frac{1}{R} \Bigg) = \frac{8\epsilon}{(1-4\epsilon^2)} \le 16\epsilon, \forall \epsilon \in \bigg(0,\frac{1}{\sqrt{8}}\bigg] \Bigg].
\end{align}

Finally noting that $\tau_A = \sum_{B \in \cB}[N_B(\tau_A)]$, and combining \eqref{eq:lb1} and \eqref{eq:lb2}, we get 

\begin{align*}
256\epsilon^2\E_{\nu^B_\sS}[\tau_A] =  \sum_{S \in \cA}\E_{\nu^B_\sS}[N_B(\tau_A)](256\epsilon^2) \ge (n-1-q)\ln \frac{q}{4\delta}.
\end{align*}
The proof now follows choosing $q = \lfloor \frac{n}{2} \rfloor$ and the fact that $n \ge 4$, as $(n - 1 - q) \ge \frac{n}{2} - 1 \ge \frac{n}{4}$ for any $n \ge 4$. Also $\ln q \ge \ln (\frac{n-1}{2}) \ge \ln \frac{n}{4}$ for any $n \ge 2$. Thus above construction shows the existence of a problem instance $\bnu = \nu^B_\sS$, such that $\E_{\nu^B_\sS}[\tau_A] = \frac{n}{1024\epsilon^2}\ln \frac{n}{16\delta} = \Omega\Big(\frac{n}{\epsilon^2}\ln \frac{n}{\delta}\Big)$, which concludes the proof.
\end{proof}

\subsection{Proof of Theorem \ref{thm:pac2_lb}}

\begin{proof}
The result can be obtained following an exact same proof as that of Theorem \ref{thm:lb_plpac_win} with the observation that for any $\theta > 0$, 
for any of the problem instances $\nu_{S}, \, S \subseteq [n-1]$, $\bsigma_{S}$ is the only \ebr\, for $\nu_S$ as follows from the observation that:

\textbf{Case 1.} For any $i \in [n-1]\sm S$
\begin{align*}
Pr_S(0|\{i,0\}) = \frac{\theta\bigg( \frac{1}{4} - \epsilon^2 \bigg)}{\theta\bigg( \frac{1}{4} - \epsilon^2 \bigg) + \theta\bigg( \frac{1}{2} - \epsilon \bigg)^2} = \frac{1}{2} + \epsilon.
\end{align*}

So $i$ must follow $0$, in the any \ebr.

\textbf{Case 2.} For any $i \in S$
\begin{align*}
Pr_S(i|\{i,0\}) = \frac{\theta\bigg( \frac{1}{2} + \epsilon \bigg)^2}{\theta\bigg( \frac{1}{4} - \epsilon^2 \bigg) + \theta\bigg( \frac{1}{2} + \epsilon \bigg)^2} = \frac{1}{2} + \epsilon.
\end{align*}

So $0$ must follow $i$, in the any \ebr. So the only choice of \ebr\, for problem instance $\bnu_S$

\end{proof}

\subsection{Proof of Theorem \ref{thm:lb_pacpl_rnk}}

\lbrnk*

\begin{proof}

The proof follows exactly following the same lines of argument as of Theorem \ref{thm:lb_plpac_win}. The only difference lies in computing the KL-divergence terms in the left hand side of Lemma \ref{lem:gar16} for TR-$m$ feedback model. We consider the exact same set of problem instances for the purpose as defined in Theorem \ref{thm:lb_plpac_win}. 

The interesting thing however to note is that how top-$m$ ranking feedback affects the KL-divergence analysis in this case. Precisely, using chain rule of the KL-divergence along the two case analyses of Eqn. \ref{eq:lb_kl0} and \ref{eq:lb_kl0}, it can be shown that with TR feedback

\begin{align*}
\nonumber KL(\nu^B_\sS, \nu^B_{\tS^*}) \le m\Big( R - \frac{1}{R}\Big)^2
\end{align*}
for any set $B \in \mathcal B$. This shows a multiplicative $m$-factor blow up in the KL-divergence terms compared to earlier case with WI, owning to top-$m$ ranking feedback--this in fact triggers the $\frac{1}{m}$ reduction in regret learning rate as shown below. Similar to \eqref{eq:lb2} we here get

\begin{align}
\label{eq:lb22}
\nonumber \sum_{a = 2}^{n}&\sum_{\{S \in \cA \mid a \in S\}} \E_{\nu^B_\sS}[N_B(\tau_A)]KL(\nu^B_\sS,\nu^B_{\tS^*}) \le m\sum_{S \in \cA}\E_{\nu^B_\sS}[N_B(\tau_A)]\Bigg( R - \frac{1}{R} \Bigg)^2\\
& \le 256m\epsilon^2\sum_{S \in \cA}\E_{\nu^B_\sS}[N_B(\tau_A)] ~~\Bigg[\text{since, } \Bigg( R - \frac{1}{R} \Bigg) = \frac{8\epsilon}{(1-4\epsilon^2)} \le 16\epsilon, \forall \epsilon \in \bigg(0,\frac{1}{\sqrt{8}}\bigg] \Bigg].
\end{align}

Now noting that $\tau_A = \sum_{B \in \cB}[N_B(\tau_A)]$, and combining \eqref{eq:lb1} and \eqref{eq:lb22}, we further get 

\begin{align*}
256m\epsilon^2\E_{\nu^B_\sS}[\tau_A] =  \sum_{S \in \cA}\E_{\nu^B_\sS}[N_B(\tau_A)](256m\epsilon^2) \ge (n-1-q)\ln \frac{q}{4\delta}.
\end{align*}

The proof now similarly follows choosing $q = \lfloor \frac{n}{2} \rfloor$ with $n \ge 4$. The above inequality implies that 
\[
\sum_{S \in \cA}\E_{\nu^B_\sS}[N_B(\tau_A)] \ge \frac{1}{256m\epsilon^2}\bigg(\frac{n}{4}\bigg)\ln \frac{n}{4\delta},
\]

which certifies existence of a problem instance $\bnu = \nu^B_\sS$, such that $\E_{\nu^B_\sS}[\tau_A] = \frac{n}{1024m\epsilon^2}\ln \frac{n}{16\delta} = \Omega\Big(\frac{n}{m\epsilon^2}\ln \frac{n}{\delta}\Big)$, and the claim follows.
\end{proof}

\section{Appendix for Section \ref{sec:res_fr}}
\label{app:res_fr}


\subsection{Proof of Theorem \ref{thm:batt_giant_fr}}

\ubantfr*

\begin{proof}
Note that the only difference of Algorithm \ref{alg:alg_ant_mod} from that of Algorithm \ref{alg:alg_ant} is former plays each group $\cG_g$ only for $\frac{1}{m}$ fraction of the later (as $t$ is set to be $t:= \frac{2k}{m\epsilon'^2}\log \frac{1}{\delta'}$ for Algorithm \ref{alg:alg_ant_mod}). The sample complexity bound of Theorem \ref{thm:batt_giant_fr} thus holds straightforwardly, same as the derivation shown in the proof of Theorem \ref{thm:batt_giant} for proving the sample complexity guarantee of Algorithm \ref{alg:alg_ant}.

The main novely lies in showing the with TR feedback how does the same guarantee of Theorem \ref{thm:batt_giant} still holds. This essentially holds due to the rank breaking updates on each pair $w_{ij}$ as formally justified below.

The proof follows exactly the same analysis till Lemma \ref{lem:ant_pivpickup} as the expression of $t$ is not used till that part. The crucial claim now is to prove an equivalent statement of Lemma \ref{lem:ant_pivpickup} for the current value of $t$. We show this as follows:

Consider any particular set $\cG_g$ at any iteration $\ell \in \lfloor \frac{n-1}{k-1} \rfloor$ and  define $q_{i}: = \sum_{\tau = 1}^{t}\1(i \in \cG^\tau_{g m})$ as the number of times any item $i \in \cG_g$ appears in the top-$m$ rankings in $t$ rounds of play of the subset $\cG_g$. Then conditioned on the event that $b$ is indeed an $\epsilon_b$-\emph{Best-Item}, for any group $g \in [G]$, then with probability at least $\Big(1-\frac{\delta}{8n}\Big)$, the empirical win count $w_{b} > (1-\eta)\frac{t}{2k}$, for any $\eta \in \big(\frac{1}{8\sqrt 2},1 \big]$. More formally,

\begin{lem}
\label{lem:divbat_n1} 
Conditioned on the event that $b$ is indeed an $\epsilon_b$-\emph{Best-Item}, for any group $g \in [G]$ with probability at least $\Big(1-\dfrac{\delta}{8n}\Big)$, $q_{b} \ge (1-\eta)\frac{mt}{2k}$, for any $\eta \in \big(\frac{1}{8\sqrt 2},1 \big]$.
\end{lem}

\begin{proof}
Fix any iteration $\ell$ and a set $\cG_g$, $g \in 1,2, \ldots G$. Define $i^\tau: = \1(i \in \cG_{g m}^\tau)$ as the indicator variable if $i^{th}$ element appeared in the top-$m$ ranking at iteration $\tau \in [t]$.  Recall the definition of {TR} feedback model (Sec. \ref{sec:feed_mod}). Using this we get $\E[b^\tau] = Pr(\{b \in \cG_{g m}^\tau\}) = Pr\big( \exists j \in [m] ~|~ \sigma(b) = j \big) = \sum_{j = 1}^{m}Pr\big( \sigma(b) = j \Big) = \sum_{j = 0}^{m-1}\frac{1}{2(k-j)} \ge \frac{m}{2k}$, since $Pr(\{b | S\}) = \frac{\theta_{b}}{\sum_{j \in S}\theta_j} \ge \frac{1/2}{|S|}$ for any $S \subseteq [\cG_g]$, as $b$ is assumed to be an $\epsilon_b$-\emph{Best-Item} where $\epsilon_b \ge \frac{1}{2} \implies \theta_b \ge 1 - \frac{1}{2} = \frac{1}{2}$. Thus we get $\E[q_{b}] = \sum_{\tau = 1}^{t}\E[b^\tau] \ge \frac{mt}{2k}$. 


Now applying Chernoff-Hoeffdings bound for $w_{b}$, we get that for any $\eta \in (\frac{3}{32},1]$, 

\begin{align*}
Pr_b\Big( q_b \le (1-\eta)\E[q_b] \Big) & \le \exp\Big(- \frac{\E[q_b]\eta^2}{2}\Big) \le \exp\Big(- \frac{mt\eta^2}{4k}\Big), ~\Big(\text{since } \E[q_{b}] \ge \frac{mt}{2k}\Big) 
 \\
& \le \exp\bigg(- \frac{\eta^2}{2\epsilon'^2} \ln \bigg( \frac{1}{\delta'} \bigg) \bigg) \le \exp\bigg(- \ln \bigg( \frac{1}{\delta'} \bigg) \bigg) \le \frac{\delta}{8n},
\end{align*}

where the second last inequality holds for any $\eta > \frac{1}{8\sqrt 2}$ as it has to be the case that $\epsilon' < \frac{1}{16}$ since $\epsilon \in (0,1)$.  Thus for any $\eta > \frac{1}{8 \sqrt 2}$, $\eta^2 \ge 4\epsilon'^2$.

Thus we finally derive that 
with probability at least $\Big(1-\frac{\delta_\ell}{8n}\Big)$, one can show that $q_{b} > (1-\eta)\E[q_{b}] \ge (1-\eta)\frac{tm}{2k}$, and
the proof follows henceforth.
\end{proof}

Above is the crucial most result due to which it is possible to prove Lemma \ref{lem:piv_pib_conf} to be true in this case as well, even with an $m$-factor reduced sample complexity. This follows since:

\begin{align}
\label{eq:piv_long2}
\nonumber Pr_b\Bigg( \bigg\{ |p_{ib}  - \hp_{ib}| > \frac{\epsilon}{16}  \bigg\} \cap \cF_g \Bigg) & \le  Pr_b\Bigg( \bigg\{ |p_{ib} - \hp_{ib}| > \frac{\epsilon}{16}  \bigg\} \cap \bigg\{ n^g_{ib} \ge \frac{mt}{4k} \bigg\} \Bigg)\\
& \le 2\exp\Big( -2\frac{mt}{4k}\big(\frac{\epsilon}{16}\big)^2 \Big)  = \frac{\delta}{4n},
\end{align}
where same as before the first inequality follows as $\cF_g \implies n^g_{ib} \ge \frac{t}{4k}$, and the second inequality holds due to Lemma \ref{lem:pl_simulator} with $\eta = \frac{\epsilon}{16}$, and $v = \frac{mt}{4k}$. 

Thus the rest of the proof can be derived following the exact same analysis as of Theorem \ref{thm:batt_giant} post Lemma \ref{lem:piv_pib_conf}, which shows that \algant \, indeed returns an \ebr\, with probability at least $(1-\delta)$ for TR feedback model, and the claim of Theorem \ref{thm:batt_giant_fr} holds good.
\end{proof}

\end{document}